%% file: multiplicity-tokenizations.tex
\documentclass{article}
\pdfoutput=1

\usepackage[utf8]{inputenc} 
\usepackage[T1]{fontenc}    
\PassOptionsToPackage{hyphens}{url} 
\usepackage[hyperfootnotes=false]{hyperref}
\usepackage{booktabs}       
\usepackage{amsfonts}       
\usepackage{nicefrac}       
\usepackage{microtype}      
\usepackage{xcolor}         
\usepackage{algorithm}
\usepackage{algpseudocode}
\usepackage{graphicx}
\usepackage{enumerate}
\usepackage{caption}
\usepackage{subcaption}
\usepackage{comment}
\usepackage[export]{adjustbox}
\usepackage{mathtools}
\usepackage{amsthm}
\usepackage{authblk}
\usepackage{fullpage}
\usepackage{multirow}
\usepackage{makecell}

\usepackage{enumitem}

\usepackage[round]{natbib}

\usepackage{tikz}
\usepackage{xcolor}
\usepackage{tcolorbox}

\usepackage{Definitions}
\usepackage{dsfont}

\newlength{\inlineheight}
\title{Tokenization Multiplicity Leads to\\Arbitrary Price Variation in LLM-as-a-service}

\author{Ivi Chatzi$^1$}
\author{Nina Corvelo Benz$^2$}
\author{Stratis Tsirtsis$^3$}
\author{Manuel~Gomez-Rodriguez$^1$}
\affil{$^{1}$Max Planck Institute for Software Systems, Kaiserslautern, Germany \\
\{ichatzi, manuel\}@mpi-sws.org}

\affil{$^{2}$Max Planck Institute of Biochemistry, Martinsried, Germany \\ corvelo@biochem.mpg.de}

\affil{$^{3}$Hasso Plattner Institute, Potsdam, Germany \\ stratis.tsirtsis@hpi.de}

\date{}

\begin{document}
\maketitle

\begin{abstract}
  \input{000abstract}
\end{abstract}

\section{Introduction}
\label{sec:intro}
\input{010introduction.tex}

\section{Preliminaries}
\label{sec:preliminaries}
\input{020preliminaries.tex}

\section{Can Hans and Emma Receive Different Tokenizations for the Same Output String?}
\label{sec:multiplicity}
\input{030multiplicity.tex}

\section{Avoiding Tokenization Multiplicity through Canonical Generation}
\label{sec:method}
\input{040method.tex}

\section{Discussion and Future Work}
\label{sec:discussion}
\input{050discussions.tex}

\section{Conclusions}
\label{sec:conclusions}
\input{060conclusions.tex}


\vspace{2mm}
\xhdr{Acknowledgements} 
Gomez-Rodriguez acknowledges support from the European Research Council (ERC) under the European Union'{}s Horizon 2020 research and innovation programme (grant agreement No. 945719).
Tsirtsis acknowledges support from the Alexander von Humboldt Foundation in the framework of the Alexander von Humboldt Professorship (Humboldt Professor of Technology and Regulation awarded to Sandra Wachter) endowed by the Federal Ministry of Education and Research via the Hasso Plattner Institute.

\clearpage

{ 
\small
\bibliographystyle{unsrtnat}
\bibliography{ref}
}

\clearpage
\newpage

\appendix

\input{070appendix.tex}

\end{document}

%% file: 000abstract.tex
Providers of LLM-as-a-service have predominantly adopted a simple pricing model: users pay a fixed price per token.
%
Consequently, one may think that the price two different users would pay for the same output string under the same input prompt is the same. In our work, we show that, surprisingly, this is not (always) true.
%
We find empirical evidence that, particularly for non-english outputs, both proprietary and open-weights LLMs often generate the same (output) string with multiple different tokenizations, even under the same input prompt, and this in turn leads to arbitrary price variation.
%
To address the problem of tokenization multiplicity, we introduce canonical generation, a type of constrained generation that restricts LLMs to only generate canonical tokenizations---the unique tokenization in which each string is tokenized during the training process of an LLM.
%
Further, we introduce an efficient sampling algorithm for canonical generation based on the Gumbel-Max trick.
%
Experiments on a variety of natural language tasks demonstrate that canonical generation is comparable to standard generation in terms of performance and runtime, and it solves the problem of tokenization multiplicity.


%% file: 010introduction.tex
Imagine you run an online service that offers AI-powered translation to help tourists seamlessly navigate websites in languages other than their mother tongue.
Behind the scenes, for each website the users visit, your service sends the website's text to a provider of an LLM and asks their model to generate the translation in the mother tongue of the user.
One day, reviewing your costs, you notice something strange: Hans and Emma, two users of your service, visited the same page on a website with geographical fun facts, requesting a German translation of the sentence
\begin{quote}
\centering
\textit{``The Acari River is a river of Minas Gerais state in southeastern Brazil''}.
\end{quote}
Both of them received the same translated text down to the last character, yet, you were charged different amounts for each by the LLM provider.
How can identical results have different costs?

The computational resources required to operate state-of-the-art large language models are too large for most (enterprise) users to run them locally~\citep{narayanan2021efficient,ziheng2024scaling}.
As a consequence, users are relying on a growing market of cloud-based providers who offer access to such models via an application programming interface (API)~\citep{sun2022blackbox, lamalfa2024language,pais2022nlp,liagkou2024cost}.
%
%
%
The way that providers of LLM-as-a-service have come to charge the use of their models is heavily dependent on one of the most distinctive technical characteristics shared by most, if not all, state-of-the-art models---they process and generate information in discrete units called tokens.\footnote{Tokens are (sub-)words, symbols and numbers that make up sentences and paragraphs.}
In particular, since the computational resources required to process an input prompt and generate a response are directly proportional to the number of tokens involved~\citep{samsi2023fromwords,fernandez2025energy}, 
providers have predominantly adopted a straightforward pay-per-token pricing model---users pay a fixed price per token.

Under the pay-per-token pricing model, one may naturally assume that, if i) two users submit the same prompt to a provider, 
ii) the provider feeds the prompt into the same model,
and iii) both users receive exactly the same output string,
then, they will pay the same price.
In our work, we show that, surprisingly, this assumption does not hold.
We find empirical evidence that, particularly for non-english outputs, both proprietary and open-weights LLMs often generate the same (output) string with different tokenizations, even under the same input prompt, and this multiplicity of tokenizations in turn leads to price variation.
Moreover, since users derive value from the text that an output token sequence represents, rather than the token sequence itself, we argue that this price variation is arbitrary and undesirable.

To address the problem of tokenization multiplicity, 
we introduce canonical generation, a type of constrained generation that restricts LLMs to only generate canonical tokenizations---the unique tokenization in which each string is tokenized during the training process of an LLM~\citep{geh2024signal}.
%
In doing so, we also introduce an efficient sampling algorithm for canonical generation based on the Gumbel-Max trick that leverages the following theoretical result, which may be of independent interest: 
to generate a canonical tokenization, a model needs to generate (partial) canonical tokenizations at each step of the generation process underpinning its functioning.

In addition to solving the problem of tokenization multiplicity, we show that, in comparison with standard generation, the distribution of token sequences generated by canonical generation is provably closer to the true distribution of token sequences used during training and, in practice, the performance and runtime of canonical generation are comparable to standard generation.
We have released all code and data used in our experiments at: \href{https://github.com/Networks-Learning/Tokenization-Multiplicity}{https://github.com/Networks-Learning/Tokenization-Multiplicity}.

\xhdr{Further related work}
Our work builds upon further related work on tokenization in LLMs, multilingual LLMs, and the economics of LLM-as-a-service.

%
The study of tokenization has a rich history in natural language processing~\citep{palmer2000tokenisation,jm3}. 
%
More recently, in the context of LLMs, there has been a renewed interest in formalizing tokenization and analyzing its properties~\citep{gastaldi2024foundations, phan2024understanding, rajaraman2025theorytokenization}, with the Byte-Pair Encoding (BPE) tokenization algorithm in particular re\-cei\-ving increased attention~\citep{Berglund2023formalizing, zouhar2024formal,kozma2024theoretical}.
The impact of tokenization on LLMs has also been studied empirically~\citep{hou2023effects,athiwaratkun2024token} in generation tasks involving foreign languages~\citep{fujii2023japanese}, translation~\citep{domingo2019tokenizationtranslation}, arithmetic~\citep{singh2024tokenizationcounts}, mental health~\citep{liu2023task}, and privacy~\citep{kharitonov2021bpememorization,petrov2023unfairnesslanguages}, among others.
%
Moreover, there is a recent line of work studying the impact of non-canonical tokenizations on text perplexity calculations~\citep{cao-rimell2021evaluate,chirkova2023marginalize,geh2024signal,vieira2024language,giulianelli2024proper},
safety guidelines~\citep{geh2025adversarialtokenization},
downstream tasks~\citep{zheng2025brokentokens},
and image watermarking~\citep{jovanović2025watermarking},
as well as an effort to circumvent non-canonical tokenizations occuring as a consequence of unusual input prompt endings~\citep{guidance2023tokenhealing}.
Within this line of work, the work most closely related to ours is by~\citet{vieira2025canonical},
who have also independently proposed canonical generation\footnote{A preliminary version of our work and their work were posted on arXiv within three days of each other.}.
However, their work does not provide empirical evidence of tokenization multiplicity and its influence on the pay-per-token pricing model as we do, 
their sampling algorithms for canonical generation are computationally less efficient compared to ours, 
and their experimental evaluation does not demonstrate that canonical generation is comparable to standard generation in terms of performance across natural language tasks. 

There has been extensive research on the capabilities of multilingual LLMs, including their performance on downstream tasks~\citep{fujii2023japanese, zhang-etal-2022-robust, rust-etal-2021-good} as well as their safety vulnerabilities~\citep{shen-etal-2024-language, wang-etal-2024-languages, deng2024multilingualjailbreakchallengeslarge, dong2024evaluatingmitigatinglinguisticdiscrimination} across different languages.
Within this research, the work most related to ours is by~\citet{ahia-etal-2023-languages}, who have shown that multilingual LLMs require more tokens to generate text of similar meaning in minority languages than in english.
Our work complements their work by providing empirical evidence that multilingual LLMs suffer from tokenization multiplicity in minority languages.

The rapidly growing literature on the economics of LLM-as-a-service~\citep{mahmood2024pricing,laufer2024fine,cai2025are,saig2024incentivizing,bergemann2025economicslargelanguagemodels,
sun2025coincountinginvisiblereasoning, sun2025invisibletokensvisiblebills,artola2025is, velasco2025auditingpaypertokenlargelanguage} has predominantly focused on understanding the incentives providers may have to act strategically at the expense of users.
Within this literature, the work most closely related to ours is by~\citet{artola2025is}, who show that token multiplicity enables an unfaithful provider to strategize and misreport the tokenization of an output generated by a model they serve without raising suspicion.
Moreover, to eliminate the incentive to strategize, they propose an incentive-compatible pay-per-character pricing model, which in turn also eliminates the price variation due to tokenization multiplicity.
However, in their work, they do not empirically demonstrate that token multiplicity can occur in practice, and it can lead to arbitrary price variation even under faithful providers, as we do in our work.
Further, it is worth noting that the pay-per-character pricing model makes a provider's profit margin vary across tokens, and this may discourage its practical adoption.

%% file: 020preliminaries.tex
In this section, we first define and formally characterize (deterministic) tokenizers and canonical tokenizations. 
Then, we briefly review the aspects of LLM training and generation that are relevant for our work.

\xhdr{Tokenizers and canonical tokenizations}
Tokenizers are tools that operate on sequences of characters (\ie, strings) and sequences of tokens, and can transform one type into the other.
Formally, let $\Sigma$ be a finite set of characters and $\Sigma^+$ be the set of all finite strings that can be created using the characters in $\Sigma$.
%
Similarly, let $V$ be a finite set of tokens, which we will refer to as the vocabulary, and $V^+$ be the set of all finite token sequences that can be created using the tokens in $V$.
Then, a tokenizer $\Tcal$ is characterized by a tuple $\Tcal \coloneq (\Sigma, V, \enc, \dec)$, where
$\enc:\Sigma^+ \rightarrow V^+$ is an encoder, which transforms strings to token sequences, and $\dec: V^+ \rightarrow \Sigma^+$ is a decoder, which transforms token sequences to strings.

Let $\sigmab$ be a string and $\sbb \in V^+$ be a token sequence such that $\dec(\sbb)=\sigmab$, then, we will say that $\sbb$ is a (valid) tokenization of the string $\sigmab$.
Here, note that there may be multiple tokenizations of a string $\sigmab$, that is, there may exist $\sbb, \sbb' \in V^+$ such that $\sbb \neq \sbb'$ and $\dec(\sbb)=\dec(\sbb')=\sigmab$.
However, given a string $\sigmab$, the encoder deterministically picks a single tokenization $\enc(\sigmab)$ among all tokenizations of $\sigmab$, which is often called the canonical tokenization~\citep{geh2024signal}.
For details on the most commonly used tokenization algorithms---BPE~\citep{gage1994bpe,sennrich2016neural}, Unigram~\citep{kudo2018subword} and Wordpiece~\citep{song2021fast}---refer to Appendix~\ref{app:tokenization-algs}.

\xhdr{LLM training and generation}
During training, an LLM learns to predict the next token in canonical sequences of tokens derived from raw text using a tokenizer. 
More formally, let $p_{\sbb}  = P[T \mid \Sbb = \sbb]$ denote the true distribution of the random variable $T\in V$, representing the next token given a (partial) token sequence $\sbb \in V^+$.
Then, the goal of LLM training is (typically) to minimize the (cross-entropy) loss between the model'{}s predicted distribution $d_{\sbb} \in \Delta(V)$ and the true next-token distribution $p_{\sbb}$.

During generation, an LLM takes as input a prompt sequence $\sbb_q \in V^+$ and responds with an output sequence $\sbb \in V^+$, generated using an autoregressive process. 
At each time step of the process, the LLM first takes as input the concatenation of the prompt sequence $\sbb_q$ and the (partial) output sequence $\sbb$, and generates a distribution over tokens $d_{\sbb_q \shortmid \sbb} \in \Delta(V)$. 
Then, it samples the next token $t$ from the distribution $d_{\sbb_q \shortmid \sbb}$ and appends the token $t$ to the output sequence $\sbb$.
The process continues until a special end-of-sequence token is sampled.
For the remainder of this paper, we will omit the prompt sequence $\sbb_q$ from the notation
and write $d_{\sbb}$ and $p_{\sbb}$ for brevity.

Importantly, since LLMs are trained on finite data, the support of the distribution $d_{\sbb}$ may differ from the respective true distribution $p_{\sbb}$.
As a result, it is possible for an LLM to generate a non-canonical token sequence, even if it has encountered no such sequences during training~\citep{cao-rimell2021evaluate, chirkova2023marginalize,vieira2024language,giulianelli2024proper,geh2025adversarialtokenization}.

%% file: 030multiplicity.tex
We answer this question affirmatively. 
Using three simple natural language tasks, we demonstrate  that tokenization multiplicity can occur in both proprietary and open-weights LLMs.
Specifically, we focus on translation (as in the example of Hans and Emma in Section~\ref{sec:intro}), spell checking, and rephrasing.
The motivation for these choices is that, in all these tasks, an LLM has to process and rewrite a given input text and is therefore more likely to generate the same output string despite the randomness in its generation.
\begin{figure}[t]
    \centering
    \includegraphics[width=\linewidth]{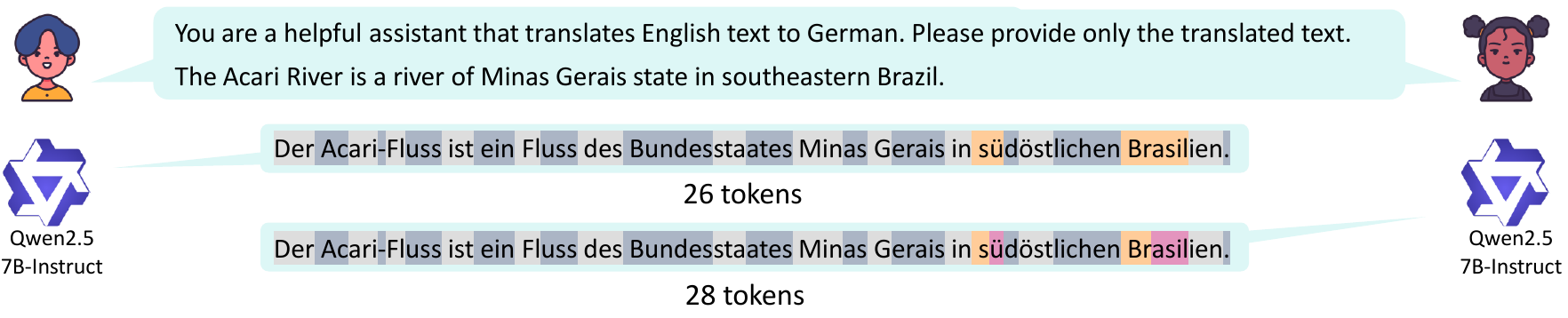}
    \caption{\textbf{Translation task example.} 
    The top box shows the input prompt, which consists of a translation instruction and the accompanying Wikipedia text to be translated.
    The latter two boxes show two outputs generated by \texttt{Qwen2.5-7B-Instruct} as response to the input prompt, corresponding to the same string but with two different tokenizations.
    }
    \label{fig:example-translate}
    \vspace{-3mm}
\end{figure}

For each of the three tasks, we construct $100$ input prompts using short texts from Wikipedia.\footnote{\href{https://dumps.wikimedia.org/}{https://dumps.wikimedia.org/}}
For the translation task, each prompt consists of a (system) instruction to translate a short text written in English to one of $5$ target languages, followed by the respective Wikipedia text in the English language.
For the other two tasks, each prompt consists of a (system) instruction to spell check or rephrase short texts written in $6$ different languages (including English) and, for the spell checking task, we first introduce a small number of typos in each Wikipedia text by 
replacing a few latin characters with other latin characters selected at random.
To simulate scenarios where different users ask an LLM to perform the same task on the same input text, we feed each input prompt to the LLM $100$ times, keeping all parameters identical except for the random seed used during 
generation.
We then focus on (the existence of) pairs of outputs whose strings are identical but their tokenization lengths differ.

As a starting point, in Figure~\ref{fig:example-translate}, we illustrate the translation task that we consider in our experiments using as an example a pair of outputs generated by $\texttt{Qwen2.5-7B-Instruct}$~\citep{yang2024qwen2} to the input text introduced at the beginning of Section~\ref{sec:intro}.
We observe that, while both output pairs correspond to the same string, Hans'{}s output consists of $26$ tokens while Emma'{}s output consists of $28$ tokens due to a difference in the generated tokenization of the words ``s{\"u}d{\"o}stlichen'' and ``Brasilien''.
Importantly, these tokenization differences cause the latter output to be $7.7\%$ more expensive than the former under pay-per-token pricing.
For similar examples illustrating the spell checking and rephrasing tasks, refer to Appendix~\ref{app:examples}.

Further, we proceed to quantify the frequency and magnitude of such inconsistencies in pricing across different tasks, LLMs, and languages.
In our experiments, we consider both proprietary LLMs from the \texttt{gpt}, \texttt{gemini}, and \texttt{claude} families and open-weights LLMs from the \texttt{Llama} and \texttt{Qwen} families.
For brevity, in the remainder of the section, we use shortened names to refer to each model, and we provide the full names of all models in Table~\ref{tab:models} in Appendix~\ref{app:exp-details}, along with additional details about the hardware, data, and APIs we use.\footnote{For (proprietary) \texttt{gpt}, \texttt{gemini}, and \texttt{claude} models, we used the official LLM-as-a-service APIs from OpenAI, Google, and Anthropic, respectively.
For open-weights models, we ran all experiments locally, however, note that a user without access to specialized hardware would normally also access them via LLM-as-a-service APIs from third-party providers.}

We first look into the (conditional) probability that, once two users receive the same output string from the model, the lengths of the two tokenizations differ.
Formally, let $\mathbf{S}$ and $\mathbf{S}'$ denote the random variables corresponding to tokenizations received by two different users as a response after providing the same prompt, \ie, $Q = Q'$.
Our goal is to estimate the quantity
$P(\texttt{len}(\mathbf{S}) \neq \texttt{len}(\mathbf{S}') \mid \texttt{dec}(\mathbf{S})=\texttt{dec}(\mathbf{S}'), Q = Q')$.
To this end, for each prompt, we count the number of output pairs whose strings match but tokenization lengths differ as a fraction of all output pairs whose strings match, and we take the average across prompts.
Note that, to perform fair comparisons across models, we intentionally focus on cases where there are differences in the tokenization length rather than the tokenization itself, since the generated tokenization is not observable under all APIs.\footnote{The API services for \texttt{gpt5m}, \texttt{gemini} and \texttt{claude} return only the output string without disclosing the tokenization that was generated. In those cases, it is still possible to identify differences in tokenization length, which is always disclosed since pricing directly depends on it.
}
Figure~\ref{fig:multiplicity-probability-de} summarizes the results for German language, which show that tokenization multiplicity 
can indeed occur both when using open-weights and proprietary models.
In particular, for all models except \texttt{gemini}, we find cases of tokenization multiplicity in at least one task and, for open-weights models (\ie, \texttt{Llama8b} and \texttt{Qwen7b}), we find that tokenization multiplicity occurs regularly across all three tasks.
Refer to Appendix~\ref{app:results-multiplicity} for qualitatively similar results in other languages.
In this context, note that the \texttt{gemini} model did generate a few ($<$$1$\%) non-canonical outputs, which suggests that this model may also present tokenization multiplicity, but we did not observe it due to the finite sample size of our experiments.\footnote{Refer to Appendix~\ref{app:results-multiplicity} for additional results showing the percentage of non-canonical outputs across all models and tasks.}

\begin{figure}[t]
  \centering
  \includegraphics[width=\linewidth]{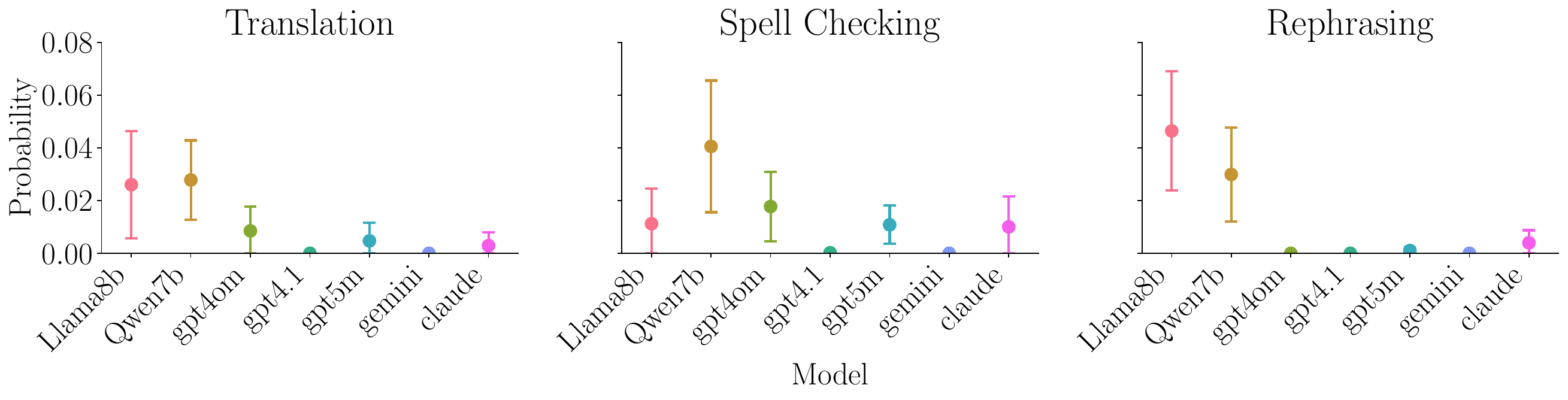}
  \caption{
  \textbf{Probability of tokenization multiplicity.}
  The plots show the empirical probability that the length of two output tokenizations to the same input prompt differ, conditioned on the output strings being the same.
  Each panel corresponds to one of the three tasks we consider in our experiments involving outputs in the German language.
  Across all panels, error bars represent $95\%$ confidence intervals resulting from $100$ input prompts.
  }
  \label{fig:multiplicity-probability-de}
\end{figure}

Next, we focus on outputs where tokenization multiplicity does occur and look into the magnitude of the variation in tokenization lengths, which directly determines the variation in prices charged to users under pay-per-token pricing.
Specifically, for each output string $\sigmab$ where there are at least two tokenizations with different length, we measure the relative difference in length (and hence price) between the shortest tokenization $\sbb_{min}$ and the longest tokenization $\sbb_{max}$, \ie, $\left(\texttt{len}\left(\sbb_{max}\right) - \texttt{len}\left(\sbb_{min}\right)\right)/\texttt{len}\left(\sbb_{min}\right)$.
 Figure~\ref{fig:rel-diff-de} summarizes the results for German language, which show that, whenever tokenization multiplicity occurs, a user may pay up to $15\%$ higher price in comparison with another user for the same output string, with similar levels of variation across both open-weights and proprietary models.
 Refer to Appendix~\ref{app:results-multiplicity} for qualitatively similar results in other languages.

Further, we analyze how the prevalence of tokenization multiplicity changes depending on the language of the output.
To this end, out of the $100$ input prompts we construct per language and task, we measure the number of prompts for which a model (here, \texttt{gpt5m}) generates at least two outputs with the same string but different tokenization lengths.
Figure~\ref{fig:conflicts-languages-gpt5} summarizes the results, which show that, across all three tasks, tokenization multiplicity is more prevalent in minority languages.
For example, even though ChatGPT officially supports languages such as Turkish or Swahili,\footnote{https://help.openai.com/en/articles/8357869} we observed that at least $7\%$ of the input prompts for these two languages led \texttt{gpt5m} to generate identical outputs with different prices.
For qualitatively similar results with other models, refer to Appendix~\ref{app:results-multiplicity}.

\begin{figure}[t]
  \centering
  \includegraphics[width=\linewidth]{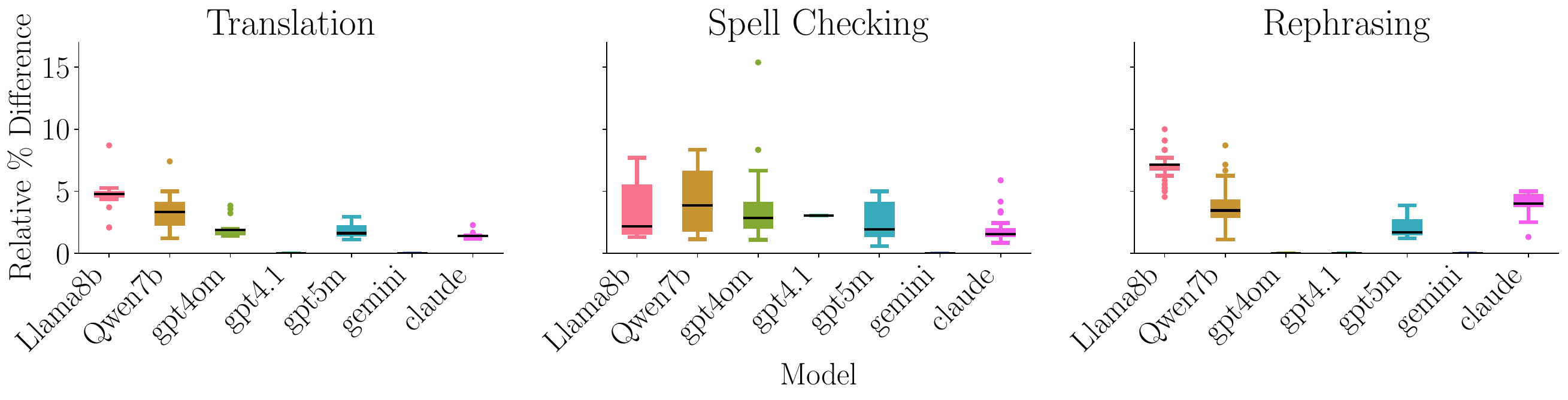}
  \caption{
  \textbf{Magnitude of price variation.}
  The plots show the empirical distribution of the relative difference in length between the longest and shortest tokenization of each output string, across all outputs where tokenization multiplicity occurs.
  Each panel corresponds to one of the three tasks we consider in our experiments involving outputs in the German language.
  Across all panels, box plots show the quartiles of the respective distributions with black horizontal lines representing median values.}
  \label{fig:rel-diff-de}
\end{figure}

To study the prevalence and magnitude of tokenization multiplicity, we have focused so far on pairs of short outputs corresponding to exactly the same string, since this allows us to attribute any price variation exclusively to token multiplicity.
In this context, one may think that, on pairs of longer outputs with multiple partial string matches, 
tokenization multiplicity may not lead to significant price variation if the difference in tokenization length across partial matches cancels out.
However, in what follows, we provide empirical evidence that the difference in tokenization length is not independent across partial matches and, as a consequence,
the price variation due to token multiplicity may be more pronounced in longer outputs.
In particular, in outputs generated by \texttt{gpt4om} to the translation task from English to German, but using longer Wikipedia texts (refer to Appendix~\ref{app:exp-details} for details regarding the experimental setup), we find that, if a word is generated with a non-canonical tokenization, then, subsequent occurrences of the word are typically generated using that same tokenization.
%
More specifically, we observe that, in $88\%$ of those cases, all subsequent occurrences (up to $10$) are generated using the same non-canonical tokenization.
%
In contrast, in $10\%$ of the remaining cases, all subsequent occurrences of the word are generated using the canonical tokenization and, in the last $2\%$ of cases, subsequent occurrences use a mix of both canonical and non-canonical tokenizations.
As an immediate consequence, it is easy to find examples of long, similar outputs containing multiple repetitions of the same word with different but consistent tokenizations, as shown in Figure~\ref{fig:example-translate-repeat} in Appendix~\ref{app:examples}.

%% file: 040method.tex
To solve the problem of tokenization multiplicity, we introduce canonical generation, a type of constrained generation that restricts LLMs to generate each output string with only its canonical tokenization.
In this context, we argue that, if one were to pick a single tokenization for each output string, the canonical tokenization presents itself as the most natural choice.
This is because, during the training of an LLM, all strings in the training data are first encoded canonically using a tokenizer, and the LLM is then trained to complete these canonical token sequences.

In what follows, we first show that, for an output token sequence generated by an LLM to be canonical, the partial token sequences generated at each step of generation must also be canonical.
Building upon this result, we then introduce an efficient and easy-to-implement sampling algorithm for canonical generation based on the Gumbel-Max trick~\citep{huijben2022review}.
Finally, we conclude by analyzing both theoretically and empirically the quality of the outputs generated by canonical generation against those generated by standard generation, as well as the runtime of our sampling algorithm.

\subsection{Subsequences of Canonical Token Sequences Must Also Be Canonical}
\label{sec:irreversibility}

In this section, we establish our main theoretical result, which shows that the most commonly used tokenizers---BPE, Unigram, and Wordpiece---are \emph{non-recovering} tokenizers.
\begin{definition} \label{definition:non-recovering}
    A tokenizer $\Tcal=(\Sigma,V,\enc,\dec)$ is called non-recovering if it holds that, for any non-canonical token sequence $\sbb\in V^+$ according to $\Tcal$ and any token $t\in V$, $\sbb \shortmid t$ is also non-canonical.
\end{definition}
More formally, we have the following theorem:\footnote{The proof for Theorem~\ref{theorem:non-canonical} can be found in Appendix~\ref{app:proof-non-recovering}. In Appendix~\ref{app:pretokenization}, we also prove that, under mild conditions, the theorem holds when using pretokenization.}
\begin{theorem} \label{theorem:non-canonical}
    BPE-, Unigram- and Wordpiece-based tokenizers are non-recovering.
\end{theorem}

\begin{figure}[t]
  \centering
  \includegraphics[width=\linewidth]{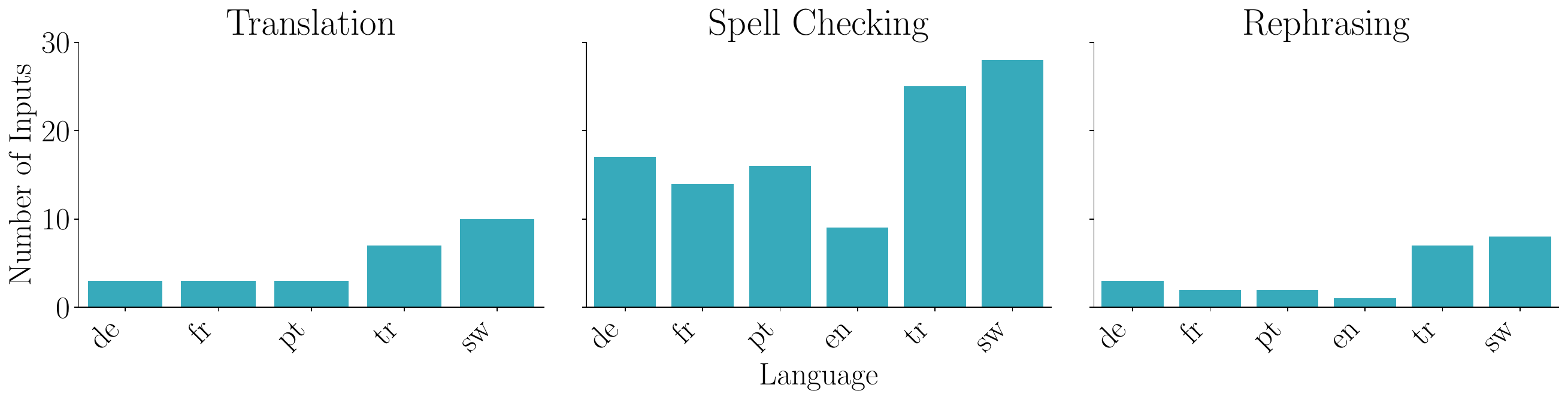}
  \caption{
  \textbf{Tokenization multiplicity across languages.}
  The plots show the number of inputs prompts for which we observe at least two outputs given by $\texttt{gpt5m}$ with the same string but different tokenization lengths.
  Each panel corresponds to one of the three tasks we consider in our experiments and pairs of letters on the x-axis correspond to different languages.
  Refer to Table~\ref{tab:languages} in Appendix~\ref{app:exp-details} for details regarding the languages we use and to Appendix~\ref{app:results-multiplicity} for qualitatively similar results using other models.
  }
  \label{fig:conflicts-languages-gpt5}
\end{figure}

This theorem immediately implies that an output token sequence is canonical if and only if the partial token sequences generated at each step of the generation process are canonical.
Moreover, this theorem also provides a plausible explanation for the empirical observation that the likelihood that an LLM generates non-canonical output sequences increases with the length of the sequence~\citep{geh2024signal}.
This is because, since sampling a ``non-canonical token'' once during the generation process is sufficient to render the output token sequence non-canonical, it is natural that the chances of this to happen increase with the number of sampled tokens.

In this context, we find it rather surprising that the above theorem holds for BPE, Unigram and Wordpiece since these (deterministic) tokenizers use fundamentally different tokenization techniques: BPE uses a rule based approach, Unigram maximizes the probability of the tokenization, and Wordpiece greedily encodes the text to minimize the number of tokens.

\subsection{An Efficient Sampling Algorithm for Canonical Generation}
\label{sec:sampling-algorithm}

Building upon Theorem~\ref{theorem:non-canonical}, we now introduce canonical generation, along with an efficient sampling algorithm to implement it.
The core principle of canonical generation is to ensure that the sampled tokens at all steps of the generation are such that the respective (partial) output remains canonical.
To this end, at each step, it sets the probability of a subset of tokens to zero---those that, when appended to a partial output sequence, would result in a non-canonical token sequence---and redistributes their probability mass to the remaining tokens proportionally to their original probability mass. 

Formally, let $d_{\sbb}$ denote the next-token distribution ge\-ne\-ra\-ted by the LLM given a partial output token sequence $\sbb$ and let $d_{\sbb}(t)$ denote the probability of sampling a token $t$ from this distribution.
Given the partial output token sequence $\sbb$, an LLM using canonical generation draws the next token in the generation process from a \emph{canonica\-lized} next-token distribution
\begin{equation} \label{eq:canonicalized-distr}
    \tilde{d}_{\sbb}(t):=\begin{cases}
        d_{\sbb}(t)/Z  & \text{if ${\sbb} \shortmid t $ is canonical} \\
        0 & \text{otherwise,}
    \end{cases}
\end{equation}
where $Z= \sum_{t \in V:\ {\sbb} \shortmid t \text{ is canonical}} d_{\sbb}(t)$ is a normalization constant that ensures that $\tilde{d}_\sbb$ is a valid probability distribution.
In that context, note that redistributing the probability mass of tokens that would lead to non-canonical token sequences proportionally to the original probabilities $d_\sbb(t)$ is a natural choice we make, inspired by other popular strategies for (stochastic) generation, such as top-$k$ and top-$p$ sampling~\citep{holtzman2019curious}, and constrained generation~\citep{beurerkellner2024guiding}.

Next, we introduce an efficient and easy-to-implement algorithm to sample from the canonicalized next-token distribution $\tilde{d}_{\sbb}$, which avoids explicitly computing the entire distribution $\tilde{d}_{\sbb}$.
The algorithm starts by sampling a value $u_t \sim \text{Gumbel}(0,1)$ from a Gumbel distribution for each token $t\in V$.
Then, it ranks the tokens in decreasing order with respect to the perturbed log-probability $\log(d_{\sbb}(t)) + u_t$. Finally, it returns the token $t$ with the largest value of $\log(d_{\sbb}(t)) + u_t$ such that ${\sbb} \shortmid t$ is canonical.
The overall procedure, summarized in Algorithm~\ref{alg:gumbel-max}, relies on a property of the Gumbel-Max trick~\citep{maddison2014sampling, huijben2022review}, which states that the $\argmax$ operation over a constrained subset of categorical outcomes is equivalent to sampling from a categorical distribution with zero probability for all outcomes outside the subset, and with the probabilities of the outcomes in the subset scaled proportionally to their original probabilities, as shown in Eq.~2 in~\citet{maddison2014sampling}.
Hence, it readily holds that Algorithm~\ref{alg:gumbel-max} returns a valid sample from the canonicalized next-token distribution $\tilde{d}_{\sbb}$ defined in Eq.~\ref{eq:canonicalized-distr}, \ie, 

\begin{equation*}
    \argmax_{t \in V:\ {\sbb} \shortmid t \text{ is canonical}} 
    \left\{ \log(d_{\sbb}(t)) + u_t \right\} \\ 
    \sim \tilde{d}_{\sbb}.
\end{equation*}

\begin{algorithm}[t]
\caption{Canonical Generation via Gumbel-Max Sampling}\label{alg:gumbel-max}
\begin{algorithmic}
    \Require next-token distribution $d_{\sbb}$
    \State $u_t \sim Gumbel(0,1)$ for all $t\in V$ 
    \For{$t\in V$ in decreasing order of $\log(d_{\sbb}(t)) + u_t$}
    \If{$s \shortmid t$ is canonical}
        \State \Return t
    \EndIf
    \EndFor
\end{algorithmic}
\end{algorithm}

Further, it is worth highlighting that, in contrast to computing the canonicalized next-token distribution $\tilde{d}_{\sbb}$, which requires evaluating the canonicity of $|V|$ token sequences, Algorithm~\ref{alg:gumbel-max} requires only a few evaluations of canonicity.
This is because, in practice, LLMs tend to generate mostly canonical token sequences~\citep{geh2024signal}, hence, the probabilities $d_{\sbb}(t)$ generated by an LLM for tokens $t$ that lead to non-canonical sequences ${\sbb} \shortmid t$ are usually small.
More specifically, let $d_{\sbb}(\text{canonical})$ be the probability mass on the subset of tokens that lead to canonical sequences, \ie, $d_{\sbb}(\text{canonical}) = \sum_{t \in V:\ {\sbb} \shortmid t \text{ is canonical}} d_{\sbb}(t)$, then Algorithm~\ref{alg:gumbel-max} requires, in expectation,
fewer than $1/d_{\sbb}(\text{canonical})$ evaluations of canonicity before successfully sampling the next token.
That is because, unlike (independent) rejection sampling from $d_{\sbb}$, which would require in expectation exactly $1/d_{\sbb}(\text{canonical})$ evaluations of canonicity until a token that leads to a canonical sequence is successfully sampled, our approach never checks the same token twice, which results in an increase in the success probability of sampling a token that leads to a canonical sequence after each failed attempt.\footnote{The number of evaluations of canonicity in rejection sampling is distributed according to a geometric distribution with success probability $d_{\sbb}(\text{canonical})$ resulting  in $1/d_{\sbb}(\text{canonical})$ evaluations in expectation until a successful sample. }

The simplest way to test whether the sequence $\sbb \shortmid t$ is canonical is to compute and check if $\enc(\dec(\sbb \shortmid t)) = \sbb \shortmid t$.
For BPE-based tokenizers however, it has been shown that it is sufficient to test if $\enc(\dec(t_{\text{last}} \shortmid t)) = t_{\text{last}} \shortmid t$, where $t_{\text{last}}$ is the final token in $\sbb$.
In fact, a recently proposed efficient algorithm to test whether $\sbb \shortmid t$ is canonical only partially applies the BPE algorithm to $\dec(t_{\text{last}} \shortmid t)$~\citep{vieira2025canonical, hayase2025sampling}.

\subsection{Performance of Canonical Generation}
\label{sec:kl-divergence}

\begin{table*}[t]
    \centering
    \begin{tabular}{l l c c c c}
    \toprule
    Task & Metric 
    & \multicolumn{2}{c}{Llama8B} 
    & \multicolumn{2}{c}{Qwen7B} \\
    \cmidrule(lr){3-4} \cmidrule(lr){5-6}
     &  & Standard & Canonical & Standard & Canonical \\
    \midrule
    \multirow{3}{*}{Translation}
        & Quality score
        & $0.72 \pm 0.02$ & $0.70 \pm 0.02$
        & $0.73 \pm 0.01$ & $0.71 \pm 0.01$ \\
        & Time per token (s)
        & $0.019$ & $0.020$ & $0.018$ & $0.019$ \\
        & Non-canonicity rate
        & $6\%$ & - & $18\%$ & - \\
    \midrule
    \multirow{3}{*}{Spell Checking}
        & $1$ $-$ edit distance
        & $0.62 \pm 0.04$ & $0.61 \pm 0.04$
        & $0.74 \pm 0.04$ & $0.72 \pm 0.04$ \\
        & Time per token (s)
        & $0.020$ & $0.023$ & $0.018$ & $0.018$ \\
        & Non-canonicity rate
        & $10\%$ & - & $19\%$ & - \\
    \midrule
    \multirow{3}{*}{Rephrasing}
        & Cosine similarity
        & $0.84 \pm 0.02$ & $0.84 \pm 0.02$ & $0.89 \pm 0.02$ & $0.88 \pm 0.02$ \\
        & Time per token (s)
        & $0.020$ & $0.020$ & $0.018$ & $0.020$ \\
        & Non-canonicity rate
        & $6\%$ & - & $5\%$ & - \\
    \midrule
    \multirow{3}{*}{MGSM}
        & Accuracy
        & $0.37 \pm 0.06$ & $0.37 \pm 0.06$ & $0.63 \pm 0.05$ & $0.62 \pm 0.06$ \\
        & Time per token (s)
        & $0.020$ & $0.021$ & $0.018$ & $0.020$ \\
        & Non-canonicity rate
        & $22\%$ & - & $29\%$ & - \\
    \bottomrule
    \end{tabular}
    \caption{{\bf Performance, (generation) time per token, and non-canonicity rate.}
    The results comprise pairs of outputs generated with standard and canonical generation in German language under the same source of randomness.
    For the time per token, confidence intervals are not shown, as they were all smaller than $10^{-4}$.
    }
    \label{tab:eval-de}
\end{table*}

We first show that, in comparison with standard generation, the distribution of tokens generated by canonical generation is provably closer to the true distribution of sequences that the LLM has seen during training.
Formally, let $p$ denote the true distribution 
over token sequences ${\sbb}\in V^+$ used during training, for which note that $p(\sbb)=0$ holds for all sequences $\sbb$ that are non-canonical.
Moreover, let $d$ denote the distribution over token sequences that the LLM generates using standard generation, and $\tilde{d}$ the distribution over token sequences that the LLM generates using canonical generation, that is, sampling from the canonicalized next-token distribution $\tilde{d}_{\sbb}$ given by Eq.~\ref{eq:canonicalized-distr} at each step of the generation process.
Then, the following theorem shows that $p$ is likely to be closer to $\tilde{d}$ than $d$ in terms of KL-divergence, a result independently established in Proposition 3 of~\citet{vieira2025canonical}:\footnote{The proof for Theorem~\ref{th:KL-divergence} can be found in Appendix~\ref{app:proof-kl-divergence}.}
\begin{theorem}\label{th:KL-divergence}
     Let $d$ be absolutely continuous\footnote{Absolute continuity is required for the KL-divergence to be well defined, \ie, we require that $d(\sbb)=0$ implies that $p(\sbb)=0$ for all $\sbb\in V^+$.} with respect to $p$.
     Moreover, assume that there exist $\sbb\in V^+$ and $t_1, t_2 \in V$ such that $\sbb \shortmid t_1$ is non-canonical with $d(\sbb\shortmid t_1)>0$ and $\sbb\shortmid t_2$ is canonical with $p(\sbb\shortmid t_2)>0$ and $d(\sbb\shortmid t_2)>0$.
     Then, it holds that
    \begin{equation}
        \textrm{KL}(p,\tilde{d}) < \textrm{KL}(p,d).
    \end{equation}
\end{theorem}

In simpler terms, the two conditions under which canonical generation brings the output token sequences closer to the true distribution are that
(i) there exist non-canonical token sequences with positive probability of being generated under $d$ so that their probability mass can be redistributed,
and (ii) there exist canonical token sequences with positive probability under $d$ and $p$ 
so that the redistribution of probability mass in $\tilde{d}$ is beneficial.
To understand the intuition behind Theorem~\ref{th:KL-divergence}, note that, by using canonical generation (\ie, sampling from $\tilde{d}$ instead of $d$),
the probability that an LLM generates non-canonical token sequences becomes zero, and the probability that it generates any other (canonical) token sequence increases
under $\tilde{d}$.
Further, since only canonical token sequences have positive probability under the true distribution $p$, this redistribution of probability mass from non-canonical token sequences to canonical ones can only bring the distribution $\tilde{d}$ closer to the true distribution $p$ compared to $d$.

On the flip side, it is important to clarify that a similar property does not necessarily hold for the respective distributions over strings. 
That is, using canonical generation, the distribution of output strings, resulting from decoding the output token sequences, is not guaranteed to be closer (in terms of KL-divergence) to the true distribution of output strings used during training.
Formally, let $p_\dec= P_{\sbb\sim p(\sbb)}[\dec(\sbb)]$ be the true distribution over strings, $d_{\dec}=P_{\sbb\sim d(\sbb)}[\dec(\sbb)]$ be the distribution of strings induced by the distribution of output token sequences $d$, and $\tilde{d}_\dec = P_{\sbb\sim \tilde{d}(\sbb)}[\dec(\sbb)]$ be the distribution of strings induced by the distribution of output token sequences $\tilde{d}$. Then, we have that 
\begin{align*}
 \textrm{KL}&(p_\dec,\tilde{d}_\dec) = 
 \\ & =\sum_{\sigmab \in \Sigma^+} p_\dec( \dec(\sbb)=\sigmab ) \ln\left(\frac{p_\dec( \dec(\sbb)=\sigmab)}{\tilde{d}_\dec( \dec(\sbb)=\sigmab )}\right)
 \\ & = \sum_{\enc(\sigmab), \sigmab \in \Sigma^+} p( \sbb=\enc(\sigmab) ) \ln\left(\frac{p( \sbb=\enc(\sigmab))}{\tilde{d}( \sbb=\enc(\sigmab))}\right)
 \\ & = \sum_{\sbb \in V^+:\, \sbb \text{ is canonical}} p( \sbb=\sbb ) \ln\left(\frac{p( \sbb=\sbb)}{\tilde{d}( \sbb=\sbb)}\right)
 \\ & =\textrm{KL}(p,\tilde{d})
\end{align*}
because there is a one-to-one mapping determined by the encoder $\enc$ from any string to a canonical token sequence, and only canonical token sequences have positive probability under $p$ and $\tilde{d}$.
In contrast, one cannot claim the same for $\textrm{KL}(p_\dec, d_\dec) $ and $\textrm{KL}(p, d)$, as the same string can have multiple tokenizations that have positive probability under $d$.
Thus, we cannot conclude that $\textrm{KL}(p_\dec,\tilde{d}_\dec) < \textrm{KL}(p_\dec,d_\dec)$.

Next, given this theoretical gap and since users typically derive value from the string that the output token sequence represents rather than the token sequence itself, we empirically compare the performance and efficiency of canonical and standard generation on the same three tasks from Section~\ref{sec:multiplicity} (\ie, translation, spell checking, and rephrasing), as well as a standard benchmark for multilingual LLM evaluation (\ie, MGSM), using (open-weights) \texttt{Llama} and \texttt{Qwen} models.\footnote{In our experiments, for standard generation, we sample from the next-token distribution $d_{\sbb}(t)$ using the default categorical sampler in \texttt{PyTorch}, which is an implementation of Gumbel-Max sampling.}
To this end, we first sample $100$ pairs of outputs per input prompt using standard and canonical generation with the same source of randomness, following~\citet{benz2025evaluationlargelanguagemodels}. 
Then, we identify the pairs in which the output generated using standard generation are non-canonical, 
which are the only ones in which standard and canonical generation differ (under Gumbel-Max based sampling),
and measure performance and time per token by (re-)sampling $10$ continuations from the token in which the output became non-canonical under standard generation using both standard and canonical generation on each corresponding output (again with the same source of randomness).

To measure performance, we use (i) a quality score provided by a pre-trained neural network for reference-free translation evaluation~\citep{guerreiro-etal-2024-xcomet} in the translation task, 
(ii) the (normalized) Levenshtein edit distance~\citep{levenshtein1966binary} between the generated text and the original text without typos in the spell checking task, 
(iii) cosine similarity of sentence embeddings~\citep{reimers-gurevych-2019-sentence} of the original and rephrased text in the rephrasing task, 
and (iv) average accuracy in the MGSM task.
Refer to Appendix~\ref{app:exp-details} for additional details regarding the experimental setup.

Table~\ref{tab:eval-de} summarizes the results for the German language, which show that both canonical and standard generation are comparable both in terms of performance and efficiency.
Here, the slightly lower performance of canonical generation can be attributed to a limitation shared by constrained generation in general, namely, occasionally restricting the sampling space to low probability generation paths~\citep{vieira2025canonical}. 
Refer to Appendix~\ref{app:results-canonical} for qualitatively similar results in other languages.

%% file: 050discussions.tex
In this section, we highlight several limitations of our work and discuss avenues for future research.

\xhdr{Tasks and languages}
Our experiments provides strong empirical evidence that tokenization multiplicity can occur on three natural language tasks, particularly in non-english languages.
However, it would be interesting to study tokenization multiplicity on additional tasks.
Moreover, 
it would be interesting to investigate whether commonly used practices to improve multilingual language generation, such as fine-tuning on different languages, using a different tokenizer per language, or using specialized models trained on mostly non-English text, may reduce the prevalence of tokenization multiplicity.

\xhdr{Methodology}
Our main theoretical result (Theorem~\ref{theorem:non-canonical}) reveals that, for BPE-, Unigram- and Wordpiece-based tokenizers, subsequences of canonical token sequences must also be canonical. 
%
It would be very interesting to better understand what property a tokenizer needs to satisfy for our result to hold. 
In this context, it would also be interesting to define relaxed notions of canonical tokenization applicable to stochastic tokenizers~\citep{kudo2018subword, provilkov2020bpe}, and adapt our theoretical result to this type of tokenizers.

Further, under canonical generation, we canonicalize the next-token distribution by redistributing the probability mass of tokens leading to non-canonical token sequences among the remaining tokens proportionally to their original probability mass.
We have shown that, in comparison with the original next-token distribution, this particular canonicalized next-token distribution leads to a distribution of output sequences that is closer to the true distribution of token sequences. 
However, we have found that, in practice, canonical generation has slightly lower performance than standard generation.
%
In future work, it would be worth to investigate global strategies beyond (next-token) sampling to redistribute the probability mass of non-canonical output token sequences to achieve better practical performance.

%% file: 060conclusions.tex
We have presented empirical evidence that, particularly for non-english outputs, both proprietary and open-weights LLMs often generate the same (output) string with different tokenizations, even under the same input prompt, and this multiplicity of tokenizations in turn leads to arbitrary, undesirable price variation.
To address the problem of tokenization multiplicity, we have proposed canonical generation, a type of constrained generation that restricts LLMs to only generate the canonical tokenization of any output string,
and introduced a simple and efficient sampling algorithm based on the Gumbel-Max trick to implement it.
Further, we have shown that, in comparison with standard generation, the distribution of token sequences generated using canonical generation is provably closer to the true distribution of token sequences used during training,
and the performance and runtime of LLMs using either method are comparable.

%% file: 070appendix.tex
\section{Tokenization Algorithms}
\vspace{-2mm}
\label{app:tokenization-algs}
\input{071app-tokenizers}
\clearpage

\section{Proof of Theorem~\ref{theorem:non-canonical}}
\label{app:proof-non-recovering}
\input{072app-proof}
\clearpage

\section{Non-recoverability under Pretokenization}
\label{app:pretokenization}
\input{073app-pretokenizer}
\clearpage

\section{Proof of Theorem~\ref{th:KL-divergence}}
\label{app:proof-kl-divergence}
\input{074app-KL-div}
\clearpage

\section{Examples of Tokenization Multiplicity}
\label{app:examples}
\input{075app-examples}
\clearpage

\section{Additional Experimental Details}
\label{app:exp-details}
\input{076app-exp-details}

\clearpage

\section{Additional Experimental Results}
\label{app:exp-results}
\input{077app-results}

%% file: 071app-tokenizers.tex
There exists many tokenization algorithms to construct the set of tokens $V$, the encoder $\enc$, and the decoder $\dec$ characterizing a tokenizer $\Tcal$.
In the following, we review three popular tokenization algorithms, BPE~\citep{gage1994bpe,sennrich2016neural}, Unigram~\citep{kudo2018subword} and Wordpiece~\citep{song2021fast}. We also discuss pretokenization, a preprocessing technique used to partition larger bodies of text before tokenization.

\vspace{-2mm}
\subsection{The BPE tokenization algorithm}
\label{app:bpe}

The BPE tokenization algorithm~\citep{gage1994bpe,sennrich2016neural} is used by most, if not all, state-of-the-art LLMs.
In a nutshell, the BPE algorithm aims to create a tokenizer $\Tcal$ with a set of tokens $V$ corresponding to character sequences that appear frequently in a training set of strings $\Ccal$.
To this end, it proceeds as follows.

In an initialization phase, the algorithm sets i) $\Sigma$ to be the set of all characters that appear at least once in $\Ccal$,
ii)
$V$ to be the set of single-character tokens, that is, for each $c \in \Sigma$, there exists one and only one $t \in V$ such that $\dec(t) = c$,
and iii)
$\Scal$ to be the set of single-character token sequences $\sbb \in V^{+}$ representing all strings in $\Ccal$.
After the initialization phase, the algorithm proceeds iteratively for a predetermined number of iterations.
At each iteration, it looks for the pair of tokens $t, t' \in V$ whose concatenation $t \shortmid t'$ appears most frequently in the set of token sequences $\Scal$,
it creates a new token $t \circ t'$, where the symbol $\circ$ denotes the merge operation and $\dec(t \circ t') = \dec(t) \shortmid \dec(t')$, 
and it adds the newly created token to $V$.
Then, for each token sequence $\sbb \in \Scal$, it replaces all occurrences of $t \shortmid t'$ by $t \circ t'$ one by one.
Lastly, it creates a merge rule $r_{t,t'}$, which specifies the replacement of $t \shortmid t'$ with $t \circ t'$, and adds it to an ordered sequence of merge rules $\Rcal$. 

After termination, the algorithm defines the encoder $\enc$ and decoder $\dec$ as follows. 
For any given token sequence $\sbb \in V^{+}$, $\dec(\sbb)$ transforms the sequence to a string one token at a time, in order, using the token definitions.
For any given string $\sigmab \in \Sigma^{+}$, $\enc(\sigmab)$ first transforms the string to a sequence of single-character tokens.
Then, it merges consecutive tokens from this sequence following the merge rules from $\Rcal$, in order, until no merge rule is applicable, and it returns the resulting sequence---the canonical sequence.\footnote{If $t \shortmid t'$ appears multiple times in a token sequence, the merge rule $r_{t, t'}$ is applied in order of appearance in the sequence.}

\vspace{-2mm}
\subsection{The Wordpiece tokenization algorithm}
The Wordpiece algorithm is similar to BPE, in the sense that it builds the token vocabulary by iteratively merging tokens.
However, the initialization phase, the merging criterion and the encoding function differ.

In the initialization phase, $\Sigma$ is set to contain all characters that appear at least once in the training set of strings $\Ccal$.
Then,
for each character $c \in \Sigma$ that appears at least once
in $\Ccal$, a single-character token $t$ is added to $V$ such that $\dec (t) = c$,
and $\Scal$ is initialized as a set that contains all single-character token sequences $\sbb \in V^+$ that represent all strings in $\Ccal$.
Interestingly, Wordpiece transforms characters (and substrings) inside words differently than characters (and substrings) at the beginning of words.
Specifically, tokens representing characters (and substrings) inside words have a special prefix.

To build the vocabulary, Wordpiece proceeds iteratively by merging existing tokens and adding them to $V$ until it reaches a predetermined size, similarly to BPE.
However, the criterion to select which pair of tokens to merge is different.
If $freq(\sbb')$ denotes the number of times that sequence $\sbb' \in V^+$ appears (as a subsequence) in the set of sequences $\Scal$,
Wordpiece looks for the pair of tokens $t, t' \in V$ that maximizes the value of $\frac{freq(t \shortmid t')}{freq(t) \cdot freq(t')}$.
Then, a new token $t \circ t'$ is added to $V$, where $\dec (t \circ t') = \dec(t) \shortmid \dec(t')$, and all occurences of $t \shortmid t'$ in each token sequence $\sbb \in \Scal$ are replaced by $t \circ t'$.
With this criterion, Wordpiece prefers to merge tokens whose concatenation appears commonly in $\Scal$, but they are not common individually.

After the above iterative process terminates, the algorithm defines the encoder and decoder functions as follows.
For any token sequence $\sbb = t_1 \shortmid \dots \shortmid t_n \in V^+$ with $n \in \NN$, the decoder returns $\dec(\sbb)=\dec(t_1) \shortmid \dots \shortmid \dec(t_n)$ using the token definitions.
Any string $\sigmab = c_1 \shortmid \dots \shortmid c_m \in \Sigma^+$, $m \in \NN$ given to the encoder is tokenized greedily from left to right, each time selecting the token in the vocabulary that represents the most characters starting from the beginning of the string.
Specifically, the first token in $\enc(\sigmab)$ is the token $t \in V$ such that $\dec(t)= c_1 \shortmid \dots \shortmid c_i$, with $i \leq m$, and $\nexists t' \in V$ such that $\dec(t')=c_1 \shortmid \dots \shortmid c_j$ with $i<j\leq m$.
In the above selection, if $c_1$ is inside a word, then $t$ must contain the special prefix.
This process continues in the same manner with the remaining string $c_{i+1}\shortmid \dots \shortmid c_m$.

\subsection{The Unigram tokenization algorithm}
The Unigram algorithm aims to create a tokenizer $\Tcal$ with a set of tokens $V$ in order to minimize a loss when tokenizing a training set of strings $\Ccal$.
In the initialization phase, $\Sigma$ is set to contain all the characters that appear at least once in $\Ccal$.
Unlike the BPE algorithm, which iteratively adds tokens to the vocabulary $V$, Unigram starts with a large vocabulary and removes tokens from it until it reaches a predetermined size.
This initial large vocabulary can be set in multiple ways, such as applying the BPE algorithm on $\Ccal$ with many iterations, or initializing it with tokens that decode to the most frequently occuring substrings in $\Ccal$.

After the initial vocabulary has been set, the algorithm proceeds in iterations, each time computing a loss over the strings in $\Ccal$ and the current vocabulary, and removing a batch of tokens from the vocabulary (typically $10\%$ or $20\%$ of tokens) whose removal minimizes
this loss.
In each iteration, every token $t$ in the current vocabulary $V$ is assigned a probability score $r(t)=\frac{freq(t)}{\sum_{t'\in V} freq(t')}$, where $freq(t)$ denotes the number of times that the token $t$ appears in all possible tokenizations of the strings in $\Ccal$.
For each token $t \in V$, the loss over the training set is computed as $\sum_{\sigmab \in \Ccal} - \log(r_{V \setminus \{t\}}(\sigmab))$,
where $r_V(\sigmab)=\max_{\sbb \in V^+, \dec(\sbb)=\sigmab}r(\sbb)$ denotes the probability score of the most likely tokenization of $\sigmab$ under vocabulary $V$, and the probability score of tokenization $\sbb=t_1 \shortmid \dots \shortmid t_n$, with $n \in \NN$, is simply $r(\sbb)=r(t_1) \dots r(t_n)$.
The tokens that minimize this loss are removed from the vocabulary and the process repeats until the vocabulary reaches a predetermined size.

After the vocabulary has been finalized, the encoder is set to tokenize a string $\sigmab \in \Sigma^+$ by finding its most likely tokenization under the final vocabulary $V$, \ie, $\enc(\sigmab)=\arg \max_{\sbb \in V^+, \dec(\sbb)=\sigmab}r(\sbb)$, using the Viterbi algorithm~\citep{viterbi1967error},
and the decoder decodes all tokens in $V$ the same way as in the original, large vocabulary.

%% file: 072app-proof.tex
In this section, we prove Theorem~\ref{theorem:non-canonical} by showing individually that each tokenization algorithm---BPE, Wordpiece, and Unigram---builds tokenizers which are non-recovering.

\subsection{BPE-based tokenizers are non-recovering}
In order to show that BPE-based tokenizers are non-recovering, we define some additional notation regarding the BPE tokenization algorithm.

When tokenizing a string $\sigmab=c_1 \shortmid \dots \shortmid c_n$, with $c_i \in \Sigma$ and $n \in \NN$, according to the BPE algorithm, we use the term \textit{merge} and write $m=(r_{t,t'},i,j)$ to refer to a single application of merge rule $r_{t,t'} \in \Rcal$ on two consecutive tokens $t \shortmid t'$ that correspond to the substring of characters $c_i \shortmid \dots \shortmid c_j$ in $\sigmab$.
To tokenize $\sigmab$, merges are performed following a unique merge sequence $M=(m_1,\dots, m_{|M|})$, where the merges are ordered $m_1 \prec \dots \prec m_{|M|}$, first by the order in which the merge rule they refer to appears in $\Rcal$, and second by position of merged token pairs in the sequence.
The notation $m \prec m'$, for $m=(r,i,j), m'=(r',i',j')$ with $r,r'\in\Rcal, \; i,j,i',j' \in [n ]$, means that either $r$ appears before $r'$ in $\Rcal$, or $r=r'$ and $i<i'$.

We now define an operator that, applied to a merge sequence $M$ that tokenizes the string ${\sigmab}$, specifies the subsequence of merges that are applied to a certain substring of $\sigmab$.
Further, we define \emph{shift equivalence}, referring to merge sequences whose merges correspond to the exact same merge rule sequence applied to different positions in a string (shifted by a constant).

\begin{definition} \label{def:merge-subseq}
    Let $\sbb=t_1 \shortmid \dots \shortmid t_{|\sbb|} \in V^{+}$ be a tokenization of $\sigmab=c_1 \shortmid \dots \shortmid c_{|\sigmab|} \in \Sigma^{+}$ obtained by applying merge sequence $M=(m_1, \dots, m_n)$. 
    For any continuous token subsequence $\sbb' $ of $\sbb$ spanning the substring $\sigmab'= c_u \shortmid \dots \shortmid c_v$, $1\leq u<v\leq|\sigmab|$,
    the operator $[M]_{\sbb'}$ denotes the subsequence of merges in $M$
    such that
    $m=(r,i,j) \in [M]_{\sbb'}$ if $m \in M$ and $u \leq i < j \leq v$.
\end{definition}

\begin{definition}\label{def:equivalence_class}
    Two merge sequences $M=(m_1, \dots, m_{|M|})$, $M'=(m'_1, \dots, m'_{|M'|})$ are \emph{shift equivalent}, denoted by  $M \shifteq M'$, if $|M|=|M'|$ and there exists $n \in \mathbb{Z}$ such that for all $i \in \{1,\dots, |M|\}$ with $m_i=(r,j,k)$, $r\in \Rcal$, $k>j>0$, it holds that $m'_i=(r,j+n,k+n)$.
\end{definition}

Before we prove that BPE-based tokenizers are non-recovering, we show that the merge sequence that creates the tokenization $\sbb=\sbb_1 \shortmid \dots \shortmid \sbb_n$ from a string $\sigmab$,
can be partitioned into $n$ disjoint (non-continuous) subsequences of merges, that create the tokenizations $\sbb_1, \dots, \sbb_n$ from the corresponding substring of $\sigmab$.

\begin{lemma} \label{lemma:merge-subseq}
    Let $\sbb \in V^{+}$ be a tokenization of $\sigmab=c_1 \shortmid \dots \shortmid c_{|\sigmab|} \in \Sigma^{+}$ obtained by applying merge sequence $M_{\sbb}$.
    For any partition $\sbb=\sbb_1 \shortmid \dots \shortmid \sbb_{n}$, where $\sbb
    _i\in V^{+}, i\in[n], n \in \NN$, the following hold:
    \begin{enumerate}
        \item For each $\sbb_i \in \sbb$, 
        there exists a merge sequence $M_{\sbb_i}$ such that applying $M_{\sbb_i}$ to the string $\dec(\sbb_i)$ creates $\sbb_i$ and $[M_{\sbb}]_{\sbb_i} \shifteq M_{\sbb_i}$.
        \item For all $\sbb_i, \sbb_j \in \sbb,
        i \neq j$,
        if $m \in [M_{\sbb}]_{\sbb_i}$ then $m \notin [M_{\sbb}]_{\sbb_j}$ and vice-versa,
        \item For each merge $m \in M_{\sbb}$ there exists $\sbb_i \in \sbb$ such that $m \in [M_{\sbb}]_{\sbb_i}$.
    \end{enumerate}
\end{lemma}

\begin{proof}
    \begin{enumerate}
        \item If  $\sbb_i$ is a tokenization of a single character $\dec(\sbb_i)=c$, then $[M_{\sbb}]_{\sbb_i}$ is the empty sequence and the statement holds trivially.
        Assume $\sbb_i$ is a tokenization of the substring $c_u \shortmid \dots \shortmid c_v$ of $\sigmab$, with $v>u>0$, and $[M_{\sbb}]_{\sbb_i}=(m_1,\dots, m_n)$.
        By Definition~\ref{def:merge-subseq}, all merges $m=(r,j_1,j_2) \in M_{\sbb}$ with $u\leq j_1 < j_2 \leq v$ belong in $[M_{\sbb}]_{\sbb_i}$, so these merges tokenize $c_u \shortmid \dots \shortmid c_v$ into $\sbb_i$.
        Then, the merge sequence $M_{\sbb_i}=(m'_1,\dots,m'_n)$, where for all $k\in[n]$ it holds that $m_{k}=(r,j_1,j_2)$ and $m'_{k}=(r,j_1-u,j_2-u)$, $r\in \Rcal$, contains the same merge rules in the same order, but with indices shifted left by $u$.
        So if $M_{\sbb_i}$ is applied to the string $\dec(\sbb_i)$ it will create $\sbb_i$.
        \item If  $[M_{\sbb}]_{\sbb_i}$ or $[M_{\sbb}]_{\sbb_j}$ are the empty sequence, meaning $\sbb_i$ or $\sbb_j$ are a tokenization of only a single character in $\sigmab$, then the statement holds trivially.
        If $\sbb_i$ is a tokenization of the substring $c_{u}\shortmid \dots \shortmid c_{v}$ and $\sbb_j$ is a tokenization of the substring $c_{u'} \shortmid \dots \shortmid c_{v'}$, since $i \neq j$ it must be that either $u<v<u'<v'$ or $u'<v'<u<v$.
        But for all $m=(r,i_1, i_2) \in [M_{\sbb}]_{\sbb_i}$ it holds that $u\leq i_1<i_2 \leq v$, and for all $m'=(r',j_1,j_2) \in [M_{\sbb}]_{\sbb_j}$ it holds that $u'\leq j_1 < j_2 \leq v'$.
        Intuitively, it is not possible for a merge to span two different subsequences $\sbb_i, \sbb_{j}$ in the partition of $\sbb$, because then (part of) $\sbb_i$ and $\sbb_{j}$ would be merged.
        \item Each merge $m=(r,j,k)\in M_{\sbb}, r\in\Rcal$ must have $1 \leq j<k \leq |\sigmab|$.
        Because the whole string $\sigmab$ is tokenized into $\sbb$ and, by definition, the token merged by $m$ cannot be part of two different subsequences in the partition, there must exist $\sbb_i\in \sbb$ that is a tokenization of a substring $c_u \shortmid \dots \shortmid c_v$ of $\sigmab$ with $u\leq j<k \leq v$.
        So by Definition~\ref{def:merge-subseq}, $m\in[M_{\sbb}]_{\sbb_i}$.
    \end{enumerate}
\end{proof}

Building up on Lemma~\ref{lemma:merge-subseq}, we show that BPE-based tokenizers are non-recovering.

\begin{lemma}\label{lemma:non-recovering-bpe}
    BPE-based tokenizers are non-recovering.
\end{lemma}
\begin{proof}
    Assume that $\sbb\shortmid t$ is canonical.
    Then, there exists a unique merge sequence $M_{\sbb\shortmid t}$ that creates it following the BPE algorithm.
    From Lemma~\ref{lemma:merge-subseq}, $M_{\sbb\shortmid t}$ can be split into $[M_{\sbb\shortmid t}]_{\sbb}$ and $[M_{\sbb\shortmid t}]_t$,
    where $[M_{\sbb\shortmid t}]_{\sbb}$ contains the merges that create $\sbb$ and $[M_{\sbb\shortmid t}]_t$ contains the merges that create $t$.
    From Lemma~\ref{lemma:merge-subseq}, there exists a merge sequence $M_{\sbb}$ that creates $\sbb$ when applied to $\dec(\sbb)$ and $[M_{\sbb\shortmid t}]_{\sbb} \shifteq M_{\sbb}$.
    Because $\sbb$ is a prefix of $\sbb \shortmid t$, the index shift is zero and we have that $[M_{\sbb\shortmid t}]_{\sbb} = M_{\sbb}$.
    
    Since $\sbb$ tokenized according to $[M_{\sbb\shortmid t}]_{\sbb} = (m_1,\dots,m_n)$ is non-canonical, there must exist a different, canonical tokenization $\sbb'\neq \sbb$ of the same character string, $\dec(\sbb)=\dec(\sbb')$.
    Let $M_{\sbb'}=(m'_1,\dots,m'_{n'})$ be the unique merge sequence that creates $\sbb'$ from $\dec(\sbb)$ according to the BPE algorithm.
    Because $M_{\sbb'}\neq [M_{\sbb\shortmid t}]_{\sbb}$, it must be that either there exists at least one $i$, $i\leq \min(n,n')$, such that $m_i \neq m'_i$, or $m_i = m'_i$ for all $i \in[\min(n,n')]$ but $n \neq n'$.
    
    We will first examine the first case. 
    Let $m_i \in [M_{\sbb\shortmid t}]_{\sbb}$ and $m'_i \in M_{\sbb'}$ be the first merges that are different between $[M_{\sbb\shortmid t}]_{\sbb}$ and $M_{\sbb'}$, meaning $\forall j<i: \; m_j=m'_j$, for $m_j \in [M_{\sbb\shortmid t}]_{\sbb}$, $m'_j \in M_{\sbb'}$.
    Because $\sbb'$ is canonical and $\sbb$ is not, it must be that 
    $m'_i \prec m_i$.
    We will now compare $M_{\sbb \shortmid t}$, $[M_{\sbb\shortmid t}]_{\sbb}$ and $M_{\sbb'}$.
    There are two sub-cases:
    \begin{enumerate}
        \item The first $i$ merges in $M_{\sbb \shortmid t}$ are the same as in $[M_{\sbb\shortmid t}]_{\sbb}$.
        This means that the first $i-1$ merges are the same as in $M_{\sbb'}$.
        Then, merge $m_i$ being applied instead of $m'_i \prec m_i$ on substring $\dec(\sbb)$, implies that $M_{\sbb \shortmid t}$ cannot be the merge sequence that creates the canonical tokenization of $\dec(\sbb \shortmid t)$ according to BPE.

        \item The first $i$ merges in $M_{\sbb \shortmid t}$ are not the same as in $[M_{\sbb\shortmid t}]_{\sbb}$.
        This means that there exists at least one merge $m \in M_{\sbb \shortmid t}$ among the first $i$ merges in $M_{\sbb \shortmid t}$ such that $m \notin [M_{\sbb\shortmid t}]_{\sbb}$.
        For any such merge $m$, as $m \notin [M_{\sbb\shortmid t}]_{\sbb}$, it must hold that $m \in [M_{\sbb\shortmid t}]_{\boldsymbol{t}}$ by  Lemma~\ref{lemma:merge-subseq}.
        So, in $M_{\sbb \shortmid t}$, merge $m_i$ is preceded by the first $i-1$ merges of $[M_{\sbb\shortmid t}]_{\sbb}$ and merge $m$.
        By Lemma~\ref{lemma:merge-subseq}, $m$ does not affect the tokens that will create $\sbb$, so the only merges in $M_{\sbb \shortmid t}$ before $m_i$ that affect  $\sbb$ are the first $i-1$ merges of $[M_{\sbb\shortmid t}]_{\sbb}$, which are the same as $M_{\sbb'}$. Then, as in case 1, merge $m_i$ being applied instead of $m'_i \prec m_i$, implies that $M_{\sbb \shortmid t}$ cannot be the merge sequence that creates the canonical tokenization of $\dec(\sbb \shortmid t)$ according to BPE.
    \end{enumerate}

    We will now examine the case where $m_i = m'_i$ for all $i \in[\min(n,n')]$, $m_i\in [M_{\sbb\shortmid t}]_{\sbb}$, $m'_i \in M_{\sbb'}$ but $n \neq n'$.
    If $n>n'$, then there exists at least one merge that can be applied on $\sbb'$ after all merges of $M_{\sbb'}$ are done, which means that $\sbb'$ cannot be canonical.
    If $n'>n$, then there exists at least one merge that can be applied on $\sbb$ after all merges of $[M_{\sbb\shortmid t}]_{\sbb}$ are done.
    This merge can also be applied on $\sbb \shortmid t$, which means that $\sbb \shortmid t$ cannot be canonical.
    
    All cases lead to a contradiction, which implies that $\sbb\shortmid t$ is non-canonical.
    We have shown that if $\sbb$ is non-canonical then $\sbb \shortmid t$ is also non-canonical. Thus,
    BPE-based tokenizers are non-recovering.
\end{proof}

\clearpage
\subsection{Unigram-based tokenizers are non-recovering.}

\begin{theorem}
    Unigram-based tokenizers are non-recovering.
\end{theorem}
\begin{proof}
    If $\sbb$ is non-canonical according to Unigram, then let $\sbb'$ denote the canonical tokenization of the same character string, $\dec(\sbb)=\dec(\sbb')$.
    Because $\sbb'$ is canonical, it must be that $r(\sbb')>r(\sbb)$.
    It follows that $r(\sbb\shortmid t)=r(\sbb)r(t) < r(\sbb') r(t) = r(\sbb'\shortmid t)$, so $\sbb\shortmid t$ cannot be the canonical tokenization of $\dec(\sbb \shortmid t)$. We have shown that if $\sbb$ is non-canonical then $\sbb \shortmid t$ is also non-canonical. Thus,
    Unigram-based tokenizers are non-recovering.
\end{proof}

\subsection{Wordpiece-based tokenizers are non-recovering.}
\label{app:recovering}

\begin{theorem}
    Wordpiece-based tokenizers are non-recovering.
\end{theorem}
\begin{proof}
    If $\sbb=t_1 \shortmid \dots \shortmid t_n$ is non-canonical according to Wordpiece, then let $\sbb'=t'_1 \shortmid \dots \shortmid t'_{n'}$ denote the canonical tokenization of the same character string, $\dec(\sbb)=\dec(\sbb')$, $n, n' \in \NN$.
    Because $\sbb' \neq \sbb$, there must exist at least one $i \leq \min(n, n')$ such that $t_i \neq t'_i$.
    It is impossible that $t_i=t'_i$ for all $i \in \min(n, n')$ but $n \neq n'$, because then $\dec(\sbb) \neq \dec(\sbb')$, as one would be a prefix of the other.
    Let $t_i, t'_i$, with $i \in \min(n, n')$ be the first different token between $\sbb$ and $\sbb'$, \ie, $\forall j<i: t_j=t'_j$ but $t_i \neq t'_i$.
    Since $\sbb'$ is canonical, it must be that $|t'_i|>|t_i|$, where
    $|t|=|\dec (t)|$
    represents the size of token $t$ based on how many characters in $\Sigma$ it encodes.
    Because $\sbb$ is a prefix of $\sbb \shortmid t$, the first $i$ tokens are the same, but $\sbb \shortmid t$ cannot be canonical because at (token) index $i$  there exists $t'_i \in V$ that encodes more characters than $t_i$, $|t'_i|>|t_i|$.
    We have shown that if $\sbb$ is non-canonical then $\sbb \shortmid t$ is also non-canonical. Thus,
    Wordpiece-based tokenizers are non-recovering.
\end{proof}

%% file: 073app-pretokenizer.tex
State-of-the-art LLMs use a tool called \emph{pretokenizer} in order to split long strings into segments that can be tokenized simultaneously and independent of each other. 
Formally, the pretokenizer is a function $\pretokenize: \Sigma^+ \rightarrow (\Sigma^+)^+$, where $(\Sigma^+)^+$ represents sequences of strings in $\Sigma^+$, such that for string $\sigmab \in \Sigma^+$, $\pretokenize(\sigmab)=(\sigmab_1,\dots,\sigmab_n)$ where $\sigmab = \sigmab_1 \shortmid \dots \shortmid \sigmab_n$, $n \in \NN$.
The encoder with pretokenizer can then be defined as $\enc_{\texttt{pre}}(\sigmab) = \enc(\sigmab_1) \shortmid \dots \shortmid \enc(\sigmab_n)$, where $\enc: \Sigma^+ \rightarrow V^+$ is an encoder based on the tokenization algorithm used in conjunction with the pretokenizer.
We extend the definition of canonical sequences to account for the effect of the pretokenizer.

\begin{definition} \label{def:pretok-canonical}
    Let $\Tcal=(\Sigma, V, \enc_{\pretokenize}, \dec)$ be a tokenizer and $\sigmab \in \Sigma^+$.
    A tokenization $\sbb \in V^+, \dec(\sbb)=\sigmab$ of $\sigmab$ is \emph{canonical} if $\sbb = \enc_{\texttt{pre}}(\sigmab)$, where $\enc(\sigmab) = \enc(\sigmab_1) \shortmid \dots \shortmid \enc(\sigmab_n)$ and $\pretokenize (\sigmab) = (\sigmab_1, \dots, \sigmab_n)$, $n\in \NN$.
\end{definition}

Pretokenizers typically work by greedily matching prefixes of a string to a regular expression, splitting when the prefix stops matching, and continuing with the remaining suffix.
If string $\sigmab$ is a match, then $\pretokenize(\sigmab)=\sigmab$, and if $\pretokenize(\sigmab)=(\sigmab_1, \dots, \sigmab_n),  n \in \NN, \sigmab \in \Sigma^+$ then
$\pretokenize(\sigmab_i)=\sigmab_i$ for all $i \in [n]$.
Regular expressions used by pretokenizers additionally satisfy a property called \emph{closed under prefix}~\citep{hayase2025sampling}, though some exceptions apply related to handling whitespace and common english contractions.

\begin{definition} \label{def:closed-under-prefix}
    A pretokenizer $\pretokenize$ is \emph{closed under prefix} if for any string $\sigmab \in \Sigma^+$ where $\pretokenize(\sigmab)=\sigmab$ and any prefix $\sigmab'$ of $\sigmab$, it holds that $\pretokenize(\sigmab')=\sigmab'$.
\end{definition}
\noindent In words, any prefix of a string that is a match to the regular expression is also a match, or equivalently, if a string is not a match then no superstring of it is a match.

We show that tokenizers with pretokenization remain non-recovering when the pretokenizer is closed under prefix.

\begin{theorem}\label{theorem:non-recovering-pretokenizer}
    Let $\Tcal=(\Sigma,V,\enc,\dec)$ be a non-recovering tokenizer and let $\texttt{pre}$ be a pretokenizer closed under prefix. 
    Then, tokenizer $\Tcal'=(\Sigma,V,\enc_{\texttt{pre}},\dec)$ is non-recovering.
\end{theorem}

\begin{proof}
    Let $\sigmab = \sigmab_1 \shortmid \dots \shortmid \sigmab_n = \dec(\sbb)$,
    where $\pretokenize(\sigmab) = (\sigmab_1, \dots, \sigmab_n)$, $\sigmab_i = \dec (\sbb_i)$ for all $i \in [n]$
    and $\sbb = \sbb_1 \shortmid \dots \shortmid \sbb_n$, $n \in \NN$.
    Additionally, let $\sigmab_t = \dec (t)$, so $\sigmab \shortmid \sigmab_t = \dec(\sbb \shortmid t)$.
    By definition~\ref{def:closed-under-prefix} and because $t$ is a single token, it must hold that either $\pretokenize(\sigmab \shortmid \sigmab_t)=(\sigmab_1, \dots, \sigmab_n, \sigmab_t)$ or $\pretokenize(\sigmab \shortmid \sigmab_t)=(\sigmab_1, \dots, \sigmab_n \shortmid \sigmab_t)$.
    For each $\sigmab_i$, it holds by definition that $\pretokenize(\sigmab_i)=\sigmab_i$, so $\enc_{\texttt{pre}}(\sigmab_i)=\enc(\dec(\sbb_i))$.

    By definition~\ref{def:pretok-canonical}, for $\sbb$ to be non-canonical there must be at least one $\sbb_i$ where $\sbb_i \neq \enc_{\texttt{pre}}(\dec(\sbb_i)) = \enc(\dec(\sbb_i))$, so $\sbb_i$ is non-canonical.
    As each substring $\sigmab_i$ is tokenized independently, if $i<n$, or if $i=n$ and $\pretokenize(\sigmab \shortmid \sigmab_t)=(\sigmab_1, \dots, \sigmab_n, \sigmab_t)$, then $\sbb_i$ is also part of $\sbb \shortmid t$, so $\sbb \shortmid t$ is non-canonical.
    Alternatively, if $i=n$ and $\pretokenize(\sigmab \shortmid \sigmab_t)=(\sigmab_1, \dots, \sigmab_n \shortmid \sigmab_t)$, then because $\Tcal$ is non-recovering, $\sbb_n$ being non-canonical implies that $\sbb_n \shortmid t$ is also non-canonical, therefore $\sbb \shortmid t$ is non-canonical.
    In all cases, it holds that if $\sbb$ is non-canonical then $\sbb \shortmid t$ is also non-canonical. Thus,
    tokenizer $\Tcal'$ is non-recovering.
\end{proof}

%% file: 074app-KL-div.tex
Here, we provide the proof of Theorem~\ref{th:KL-divergence}, which we restate below.

\xhdr{Theorem~\ref{th:KL-divergence}}
     Let $d$ be absolutely continuous
     with respect to $p$.
     Moreover, assume that there exist $\sbb\in V^+$ and $t_1, t_2 \in V$ such that $\sbb \shortmid t_1$ is non-canonical with $d(\sbb\shortmid t_1)>0$ and $\sbb\shortmid t_2$ is canonical with $p(\sbb\shortmid t_2)>0$ and $d(\sbb\shortmid t_2)>0$.
     Then, it holds that
    \begin{equation}
        \textrm{KL}(p,\tilde{d}) < \textrm{KL}(p,d).
    \end{equation}

\begin{proof}
    Assume there exists $\hat{\sbb}\in V^+$, $t_1, t_2 \in V$ such that $\hat{\sbb} \shortmid t_1$ is non-canonical and $d(\hat{\sbb} \shortmid t_1)>0$ and $\hat{\sbb}\shortmid t_2$ is canonical and $p(\hat{\sbb} \shortmid t_2)>0$ and $d(\hat{\sbb} \shortmid t_2)>0$.
    Given any token sequence $\sbb \in V^+$, let $p_\sbb= P[T | \Sbb=\sbb]$ be the true next token distribution and $d_\sbb$, $\tilde{d}_\sbb$ be the next token distribution and canonicalized next token distribution given by the LLM.
    Then, $d_{\hat{\sbb}}(t_1\,)>0$. 
    Then, we have that $Z= \sum_{t \in V:\ \hat{\sbb} \shortmid t \text{ is canonical}} d_{\hat{\sbb}}(t)<1$. By definition of $\tilde{d}_{\hat{\sbb}}$, this implies that for all $t \in V$ such that $\tilde{d}_{\hat{\sbb}}(t)>0$, we have that
\vspace{-1mm}
    \begin{equation}\label{eq:smaller_1}
        \frac{\tilde{d}_{\hat{\sbb}}(t)}{d_{\hat{\sbb}}(t)}>1
    \end{equation}
    Note that, because $\hat{\sbb}$ is canonical (by Theorem~\ref{theorem:non-canonical} and because $p(\hat{\sbb} \shortmid t_2)>0$) and $d(\hat{\sbb})>0$, it implies that $\hat{\sbb}$ also has positive probability under $\tilde{d}$, \ie, $\tilde{d}(\hat{\sbb})>0$. In particular, by definition of $\tilde{d}$ we know that $ \tilde{d}(\hat{\sbb})/ d(\hat{\sbb})\geq 1$ and thus using Eq.~\ref{eq:smaller_1} it follows that for any $t$ such that $\tilde{d}_{\hat{\sbb}}(t)>0$, $\tilde{d}(\hat{\sbb}\shortmid t)>0$ and 
\vspace{-1mm}
    \begin{equation}\label{eq:smaller_full}
        \frac{\tilde{d}(\hat{\sbb}\shortmid t)}{d(\hat{\sbb}\shortmid t)} 
        = \frac{\tilde{d}(\hat{\sbb})}{d(\hat{\sbb})} \cdot  \frac{\tilde{d}_{\hat{\sbb}}(t)}{d_{\hat{\sbb}}(t)} 
        >1
    \end{equation}
    We show that the difference in KL-divergence of $p$ from $d$ and $p$ from $\tilde{d}$ is greater than zero. First, we rewrite the difference as follows:
\vspace{-2mm}
    \begin{align}
        \textrm{KL}(p,d)-\textrm{KL}(p,\tilde{d})
        =& \sum_{\sbb\in V^+} p(\sbb) \log\left(\frac{p(\sbb)}{d(\sbb)}\right) - \sum_{\sbb\in V^+} p(\sbb) \log\left(\frac{p(\sbb)}{\tilde{d}(\sbb)}\right) \nonumber
        \\ =& \sum_{\sbb\in V^+} p(\sbb) \left[ \log\left(\frac{p(\sbb)}{d(\sbb)}\right) -  \log\left(\frac{p(\sbb)}{\tilde{d}(\sbb)}\right)   \right] \nonumber
        \\=& \sum_{\sbb \in V^+} p(\sbb) \log\left(\frac{\tilde{d}(\sbb)}{d(\sbb)}\right) \label{eq:KL_diff}
        \\=& \sum_{\sbb \in V^+:\ \tilde{d}(\sbb)>0} p(\sbb) \log\left(\frac{\tilde{d}(\sbb)}{d(\sbb)}\right) \label{eq:positive_tilde_d}
    \end{align}
    where the first equations follow from simple manipulations and Eq.~\ref{eq:positive_tilde_d} follows from the following argument. 
    Whenever $\tilde{d}(\sbb)=0$, it implies that either $d(\sbb)=0$ or that $\sbb$ is non-canonical. Both cases imply that $p(\sbb)=0$ (either by absolute continuity or non-canonicity). Whenever $p(\sbb)$ and $\tilde{d}(\sbb)$ are zero, the contribution of the corresponding term in Eq.~\ref{eq:KL_diff} is interpreted as zero because
        $\lim_{x\to 0^+} x \log x = 0.$
    
    We can break up Eq.~\ref{eq:positive_tilde_d} into two types of summand. For any $\sbb\neq \hat{\sbb}\shortmid t, t\in V$ and $\tilde{d}(\sbb)>0$, it readily follows from the definition of $\tilde{d}_\sbb$ that 
\vspace{-2mm}
    \begin{equation*}
        p(\sbb) \log\left(\frac{\tilde{d}(\sbb)}{d(\sbb)}\right) \geq p(\sbb) \log(1)=0  
    \end{equation*}
    For any $\sbb= \hat{\sbb}\shortmid t, t\in V$ and $\tilde{d}(\sbb)>0$ and $p(\sbb)>0$, it follows from Eq.~\ref{eq:smaller_full} that 
\vspace{-1mm}
    \begin{equation*}
        p(\sbb) \log\left(\frac{\tilde{d}(\sbb)}{d(\sbb)}\right) > p(\sbb) \log(1)=0  
    \end{equation*}
    Thus, we can conclude that, as there exist $t_2\in V $ such that $p(\hat{\sbb}\shortmid t_2)>0$ and $\tilde{d}(\hat{\sbb}\shortmid t_2)>0$,
    \begin{equation*}
        \textrm{KL}(p,d)-\textrm{KL}(p,\tilde{d}) > 0.
    \end{equation*} 
\end{proof}

%% file: 075app-examples.tex
\begin{figure}[h!]
\centering
\begin{subfigure}{\linewidth}
  \centering
  \includegraphics[width=\linewidth]{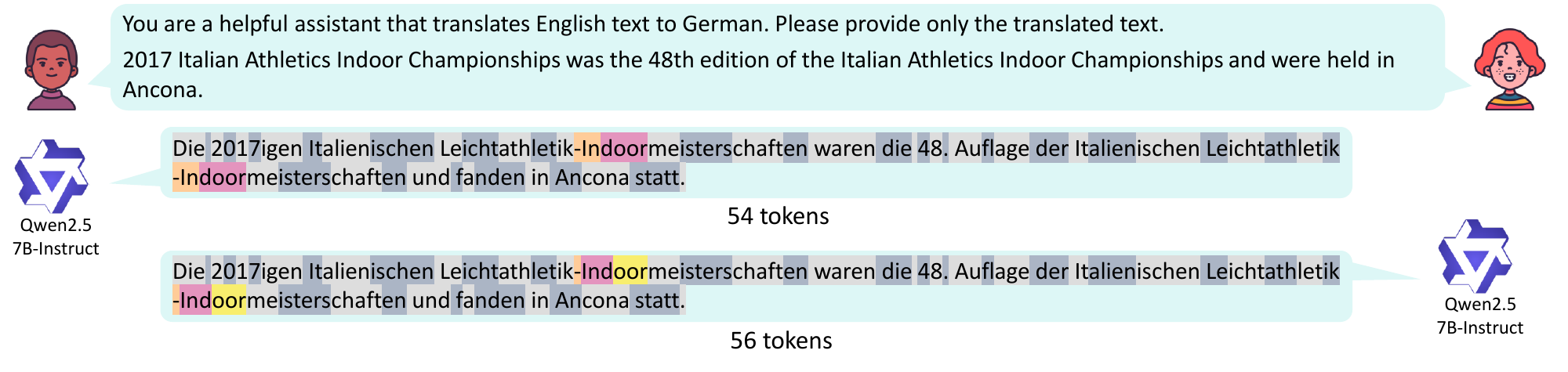}
  \caption{}
  \label{fig:example-translate-app}
\end{subfigure}
\begin{subfigure}{\linewidth}
  \centering
  \includegraphics[width=\linewidth]{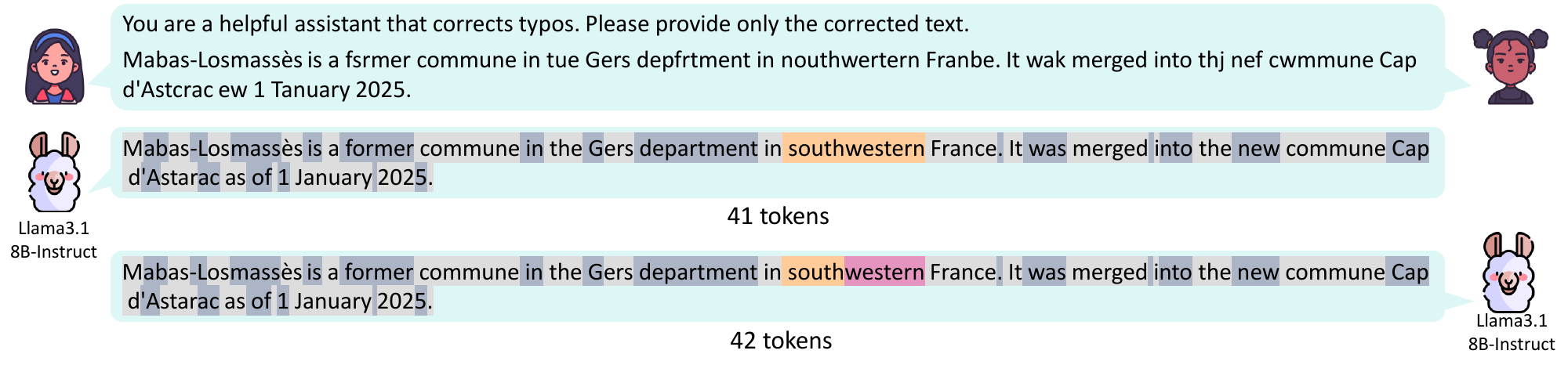}
  \caption{}
  \label{fig:example-typo}
\end{subfigure}
\begin{subfigure}{\linewidth}
  \centering
  \includegraphics[width=\linewidth]{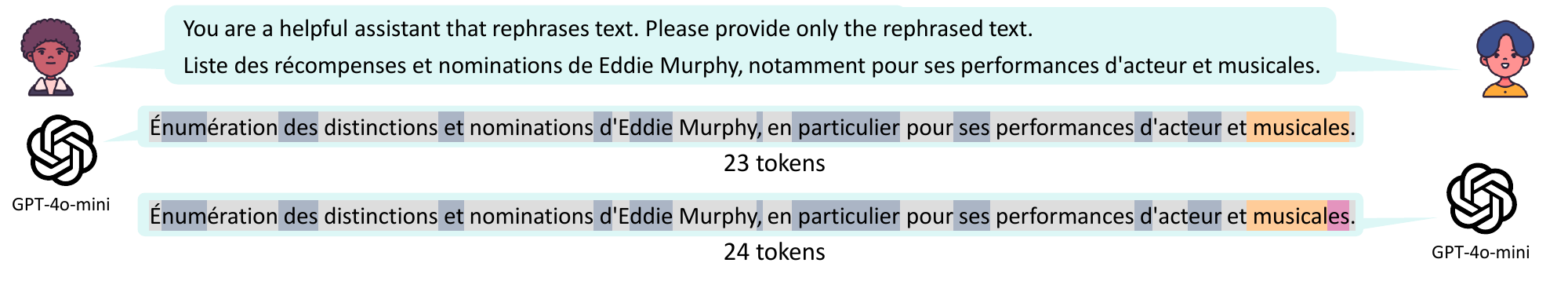}
  \caption{}
  \label{fig:example-rephrase}
\end{subfigure}

\caption{
\textbf{Examples of tokenization multiplicity in (a) translation, (b) spell checking, and (c) rephrasing.}
In each example, the top box shows the input prompt, which consists of an instruction of the task and the accompanying Wikipedia text to be processed.
The latter two boxes show two outputs generated by (a) \texttt{Qwen2.5-7B-Instruct}, (b) \texttt{Llama3.1-8B-Instruct} and (c) \texttt{GPT-4o-mini} as a response to the input prompt, corresponding to the same string but with two different tokenizations.}
\label{fig:examples}
\end{figure}

\begin{figure}[h!]
    \centering
    \includegraphics[width=0.92\linewidth]{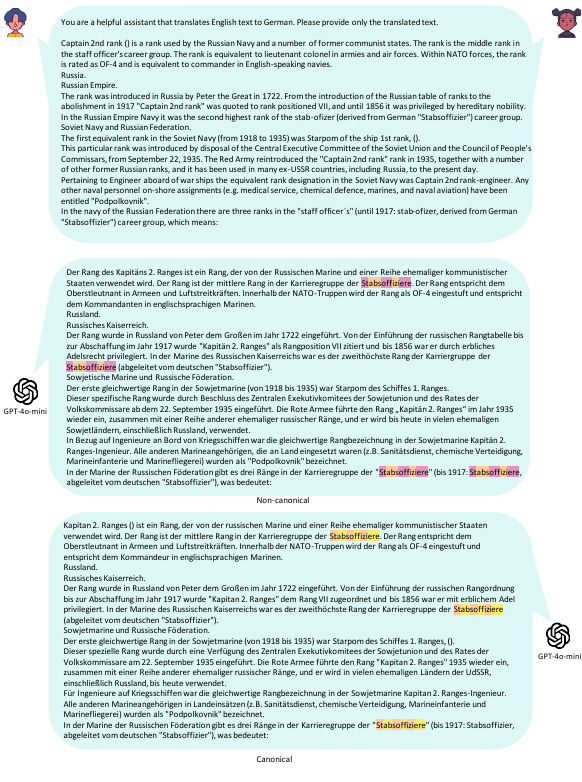}
    \vspace{-2mm}
    \caption{\textbf{Example of tokenization multiplicity in long, similar outputs.}
    The top box consists of a translation instruction and the accompanying Wikipedia text to be translated.
    The latter two boxes show two outputs generated by \texttt{gpt-4o-mini} as response to the input prompt, corresponding to the similar strings but with two different tokenizations for the word ``Stabsoffiziere".}
    \label{fig:example-translate-repeat}
\end{figure}

%% file: 076app-exp-details.tex
\xhdr{Hardware setup}
Our experiments using open-weights models are executed on a compute server equipped with 2 × Intel Xeon Gold 5317 CPU, 1,024 GB main memory, and 2 × A100 Nvidia Tesla GPU (80 GB, Ampere Architecture).
In each experiment a single Nvidia A100 GPU is used.

\xhdr{Datasets and languages}
As input texts for our experiments on the translation, spell checking and rephrasing tasks, we used articles from the most recent Wikipedia dumps\footnote{\href{https://dumps.wikimedia.org/}{https://dumps.wikimedia.org/}} as of December 3rd 2025 in different languages.
See Table~\ref{tab:languages} for a full list of languages and the shortened names used in our plots in section~\ref{sec:multiplicity} and Appendix~\ref{app:exp-results}.
We extracted plain text from the articles
using the \texttt{wikiextractor} tool\footnote{\href{https://github.com/attardi/wikiextractor}{https://github.com/attardi/wikiextractor}}
and sampled $100$ articles from each language with length between $30$ and $300$ characters.
For the experiment with longer outputs at the end of section~\ref{sec:multiplicity}, we sampled $100$ articles in the english language with length between $1000$ and $3000$ characters.
For the spell checking task, we randomly replaced, with probability $10\%$, lowercase latin characters in the input with a random different lowercase latin character.
Additionally, for the experiments on the MGSM task in section~\ref{sec:kl-divergence} we used the MGSM benchmark~\citep{shi2022languagemodelsmultilingualchainofthought}, consisting of 250 grade-school maths problems translated in different languages, but we only considered languages in latin script.
The reason our experiments are solely focused on languages using the latin script is that for most non-latin scripts tokens very often encode at most one character, therefore many, if not all, output strings cannot be generated under multiple tokenizations.

\xhdr{Models and parameters} Table~\ref{tab:models} lists the models used in our experiments, as well as the shortened names used in our results in sections~\ref{sec:multiplicity},~\ref{sec:kl-divergence} and Appendix~\ref{app:exp-results}.
All inferences were performed with temperature set to $1.0$.
The system prompts used in the MGSM task were adopted from an open-source evaluation library\footnote{\href{https://github.com/openai/simple-evals/blob/main/mgsm\_eval.py}{https://github.com/openai/simple-evals/blob/main/mgsm\_eval.py}}, which uses a $0$-shot chain-of-thought prompting technique,
while the system prompts for the other tasks are shown in Tables~\ref{tab:system-prompt-translation},~\ref{tab:system-prompt-typos}, and~\ref{tab:system-prompt-rephrasing}.
For \texttt{gpt-5-mini}, we used the minimum reasoning setting and subtracted from the output token count any reasoning tokens that are not visible to the user.

\xhdr{API details}
To investigate tokenization multiplicity in proprietary models, we used the publicly available official API services from OpenAI, Google, and Anthropic.
Further, to measure the canonicity of an output, its tokenization must be disclosed by the API and a tokenizer must be publically available.\footnote{OpenAI provide a public tokenizer: \href{https://github.com/openai/tiktoken}{https://github.com/openai/tiktoken}}
However, the API services for \texttt{gpt5m}, \texttt{gemini} and \texttt{claude} return only the output string and number of generated tokens, without disclosing the exact tokenization.
For these models, we can identify some cases of non-canonicity, when the number of generated tokens does not match the number of tokens in the canonical tokenization.

\xhdr{Reproducibility}
We have released all code and data required to reproduce our results at the following repository: \href{https://github.com/Networks-Learning/Tokenization-Multiplicity}{https://github.com/Networks-Learning/Tokenization-Multiplicity}.
However, the exact outputs of the proprietary models are not always reproducible.
Specifically, the API services for \texttt{gemini}, \texttt{claude} and \texttt{gpt5m} do not allow setting a random seed for deterministic outputs, while for \texttt{gpt4om} and \texttt{gpt4.1} setting a random seed is possible but the output is deterministic only if it is accompanied by the same \texttt{system\_fingerprint} field, which cannot be controlled by the user.
Therefore, we have included in the repository all outputs from these models where we observed tokenization multiplicity, and believe that one can obtain qualitatively similar results by running our code.

\newpage
\begin{table}[h]
    \centering
    \begin{tabular}{lc}
        Full name & Shortened name\\
        \midrule
         German & de \\
         French & fr \\
         Portuguese & pt \\
         English & en \\
         Turkish & tr \\
         Swahili & sw \\
    \end{tabular}
    \caption{Languages used in our experiments.}
    \label{tab:languages}
\end{table}

\begin{table}[h]
    \centering
    \begin{tabular}{ll}
        Full name & Shortened name\\
        \midrule
         \texttt{Llama-3.1-8B-Instruct} & \texttt{Llama8B} \\
         \texttt{Qwen2.5-7B-Instruct} & \texttt{Qwen7B} \\
         \texttt{gpt-4o-mini} & \texttt{gpt4om} \\
         \texttt{gpt-4.1} & \texttt{gpt4.1} \\
         \texttt{gpt-5-mini} & \texttt{gpt5m} \\
         \texttt{gemini-2.5-flash-lite} & \texttt{gemini} \\
         \texttt{claude-3-haiku-20240307} & \texttt{claude} \\
    \end{tabular}
    \caption{Models used in our experiments.}
    \label{tab:models}
\end{table}

\begin{table}[h]
\centering

\begin{tcolorbox}[
    colframe=white,      
    colback=gray!14,     
    boxrule=0.5mm,       
    arc=4mm,             
    left=3mm,            
    right=3mm,           
    top=3mm,             
    bottom=3mm           
]
\begin{tabular}{ m{0.97\textwidth} }
    \textbf{System:} You are a helpful assistant that translates LANG-1 text to LANG-2. Please provide only the translated text.
\end{tabular}
\end{tcolorbox}
\caption{System prompt used for the translation task. LANG-1 and LANG-2 correspond to full names of languages from Table~\ref{tab:languages}.
}
\label{tab:system-prompt-translation}
\end{table}

\begin{table}[h!]
\centering
\begin{tcolorbox}[
    colframe=white,      
    colback=gray!14,     
    boxrule=0.5mm,       
    arc=4mm,             
    left=3mm,            
    right=3mm,           
    top=3mm,             
    bottom=3mm           
]
\begin{tabular}
{ m{0.97\textwidth} }
    \textbf{System:} You are a helpful assistant that corrects typos. Please provide only the corrected text.
\end{tabular}
\end{tcolorbox}
\caption{System prompt used for the spell checking task.}
\label{tab:system-prompt-typos}
\end{table}

\begin{table}[h!]
\centering
\begin{tcolorbox}[
    colframe=white,      
    colback=gray!14,     
    boxrule=0.5mm,       
    arc=4mm,             
    left=3mm,            
    right=3mm,           
    top=3mm,             
    bottom=3mm           
]
\begin{tabular}{ m{0.97\textwidth} }
    \textbf{System:} You are a helpful assistant that rephrases text. Please provide only the rephrased text.
\end{tabular}
\end{tcolorbox}
\caption{System prompt used for the rephrasing task.}
\label{tab:system-prompt-rephrasing}
\end{table}

%% file: 077app-results.tex
\subsection{Additional Experimental Results on Tokenization Multiplicity}
\label{app:results-multiplicity}

\begin{figure}[h!]
    \centering
    \begin{subfigure}{0.95\linewidth}
        \centering
        \includegraphics[width=0.75\linewidth]{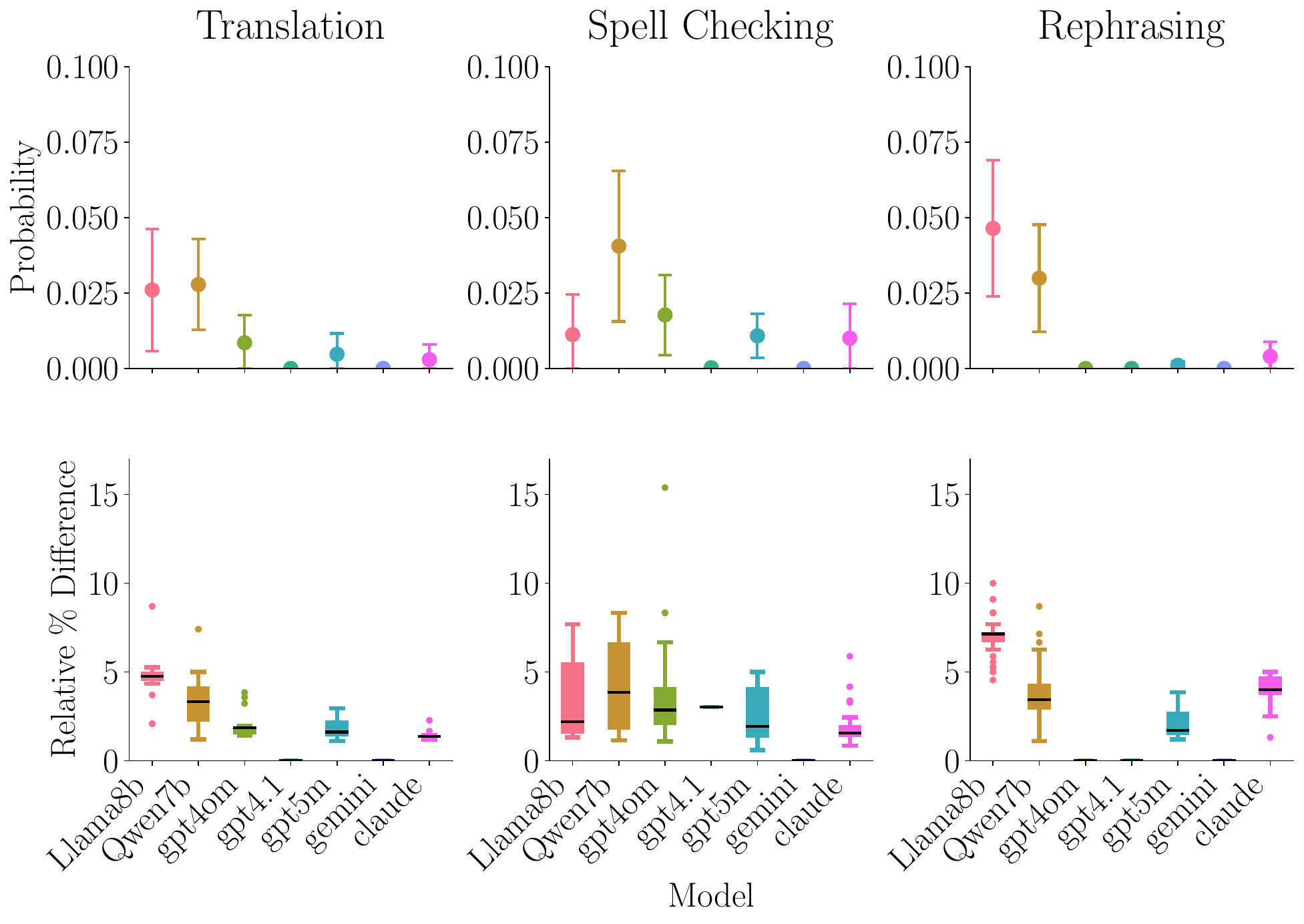}
        \caption{German}
        \label{fig:multiplicity-all-de}
    \end{subfigure}
    \begin{subfigure}{0.95\linewidth}
        \centering
        \includegraphics[width=0.75\linewidth]{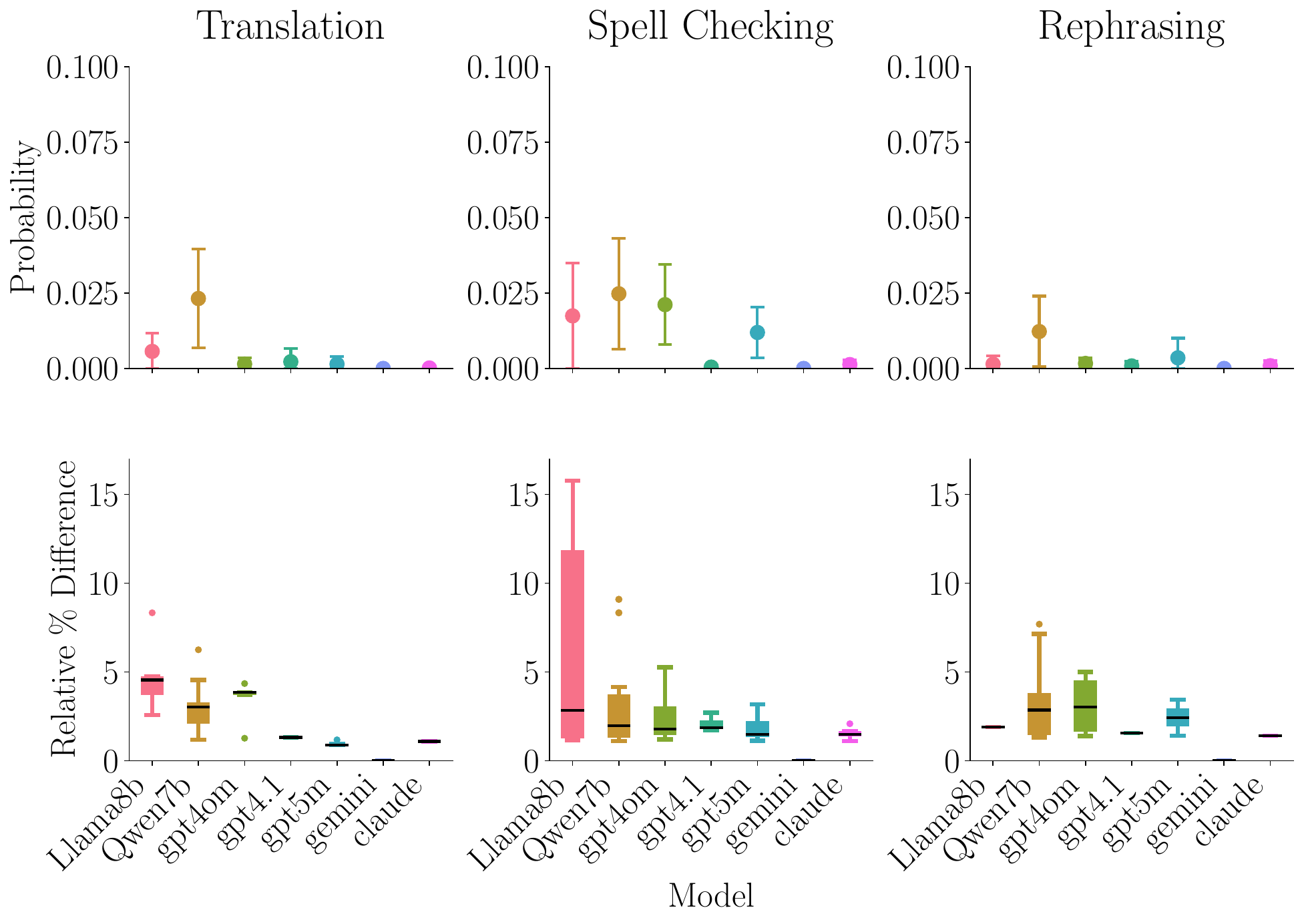}
        \caption{French}
        \label{fig:multiplicity-all-fr}
    \end{subfigure}
    
    \caption{\textbf{Probability of tokenization multiplicity and magnitude of price variation for tasks in the (a) German and (b) French languages.}}
    \label{fig:multiplicity-all-de-fr}
\end{figure}

\begin{figure}[h!]
    \centering
    \begin{subfigure}{1\linewidth}
        \centering
        \includegraphics[width=0.75\linewidth]{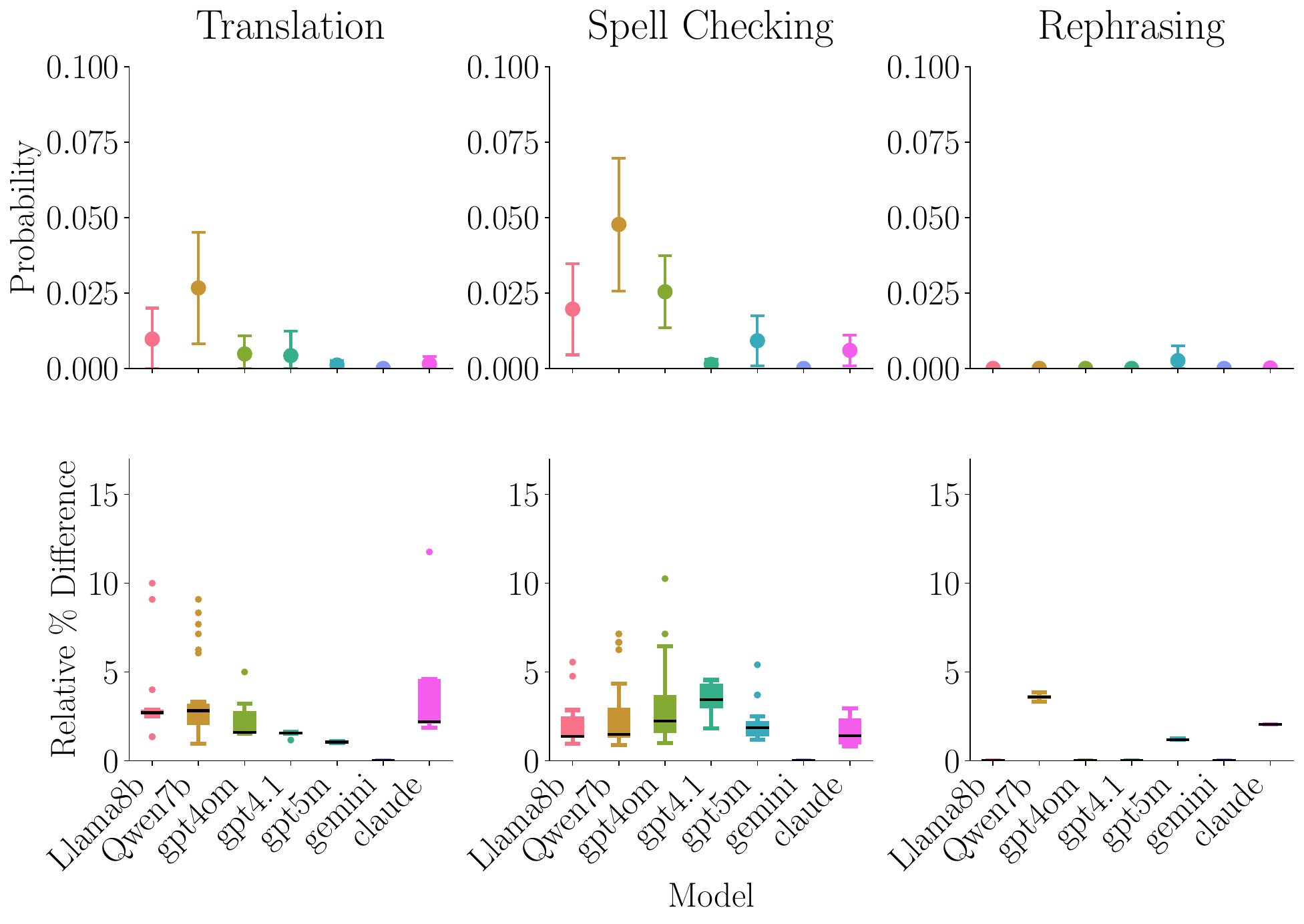}
        \caption{Portuguese}
        \label{fig:multiplicity-all-pt}
    \end{subfigure}
    \begin{subfigure}{1\linewidth}
        \centering
        \includegraphics[width=0.75\linewidth]{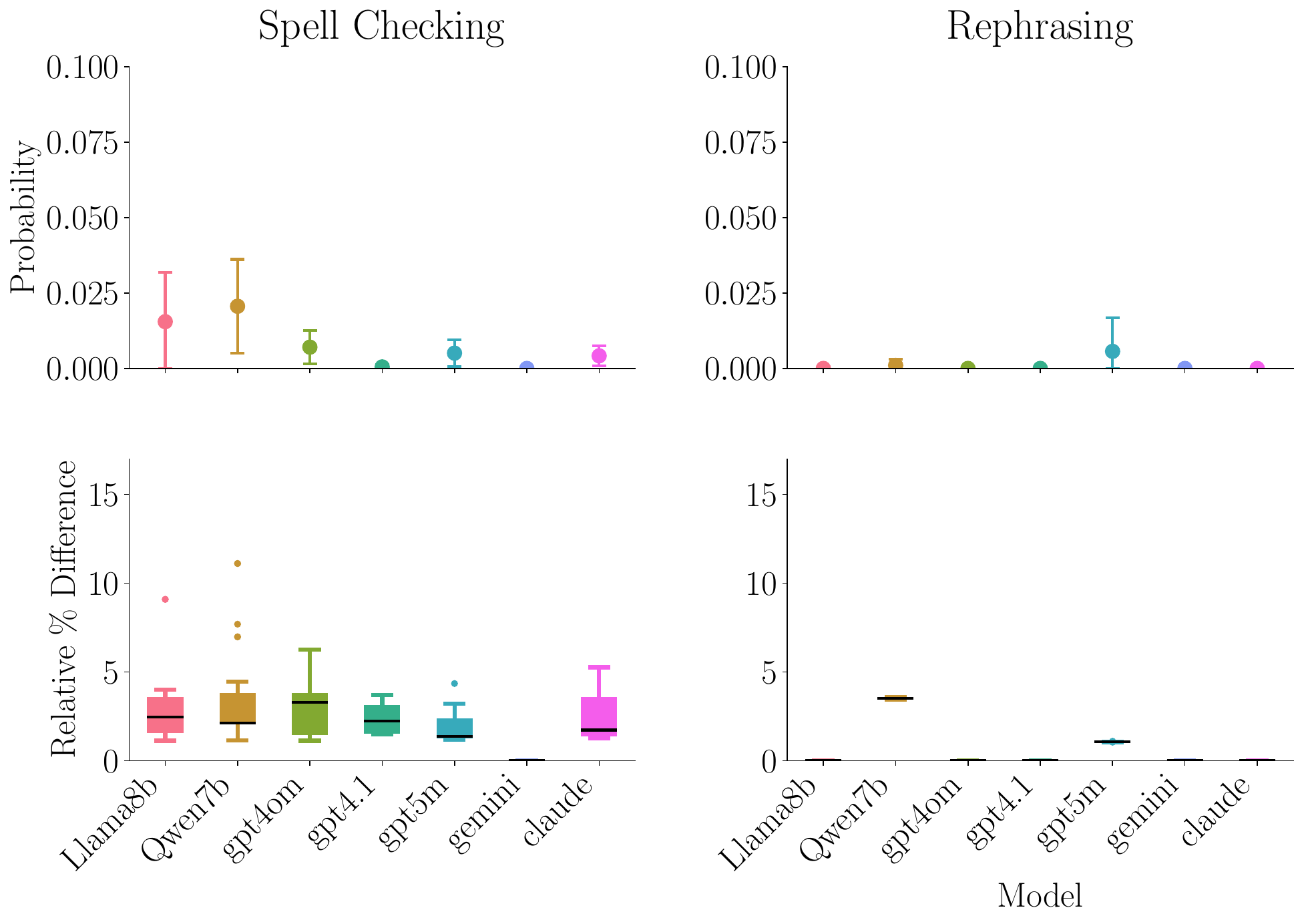}
        \caption{English}
        \label{fig:multiplicity-all-en}
    \end{subfigure}
    
    \caption{\textbf{Probability of tokenization multiplicity and magnitude of price variation for tasks in the (a) Portuguese and (b) English languages.}}
    \label{fig:multiplicity-all-pt-en}
\end{figure}

\begin{figure}[h]
    \centering
    \begin{subfigure}{1\linewidth}
        \centering
        \includegraphics[width=0.75\linewidth]{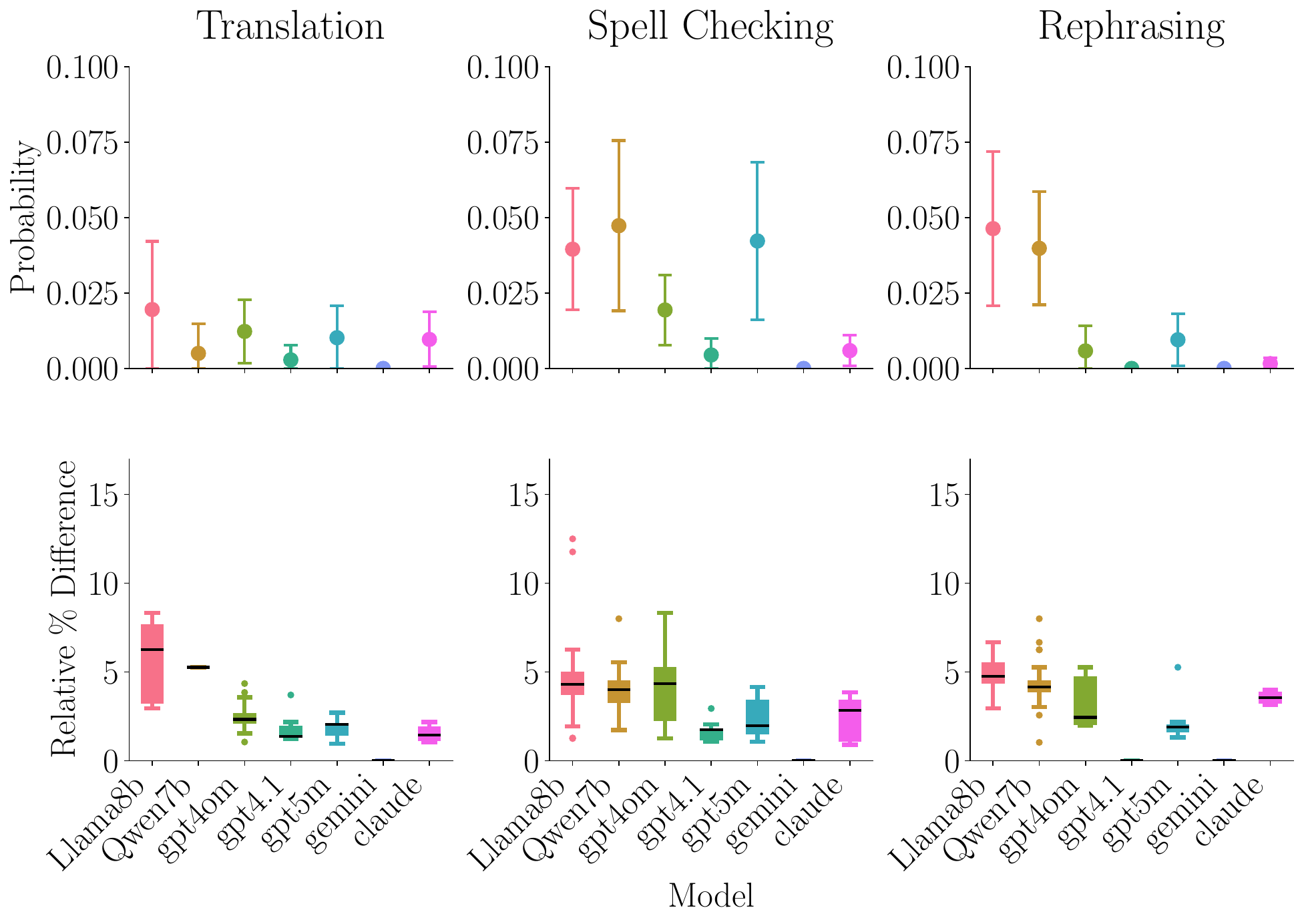}
        \caption{Turkish}
        \label{fig:multiplicity-all-tr}
    \end{subfigure}
    \begin{subfigure}{1\linewidth}
        \centering
        \includegraphics[width=0.75\linewidth]{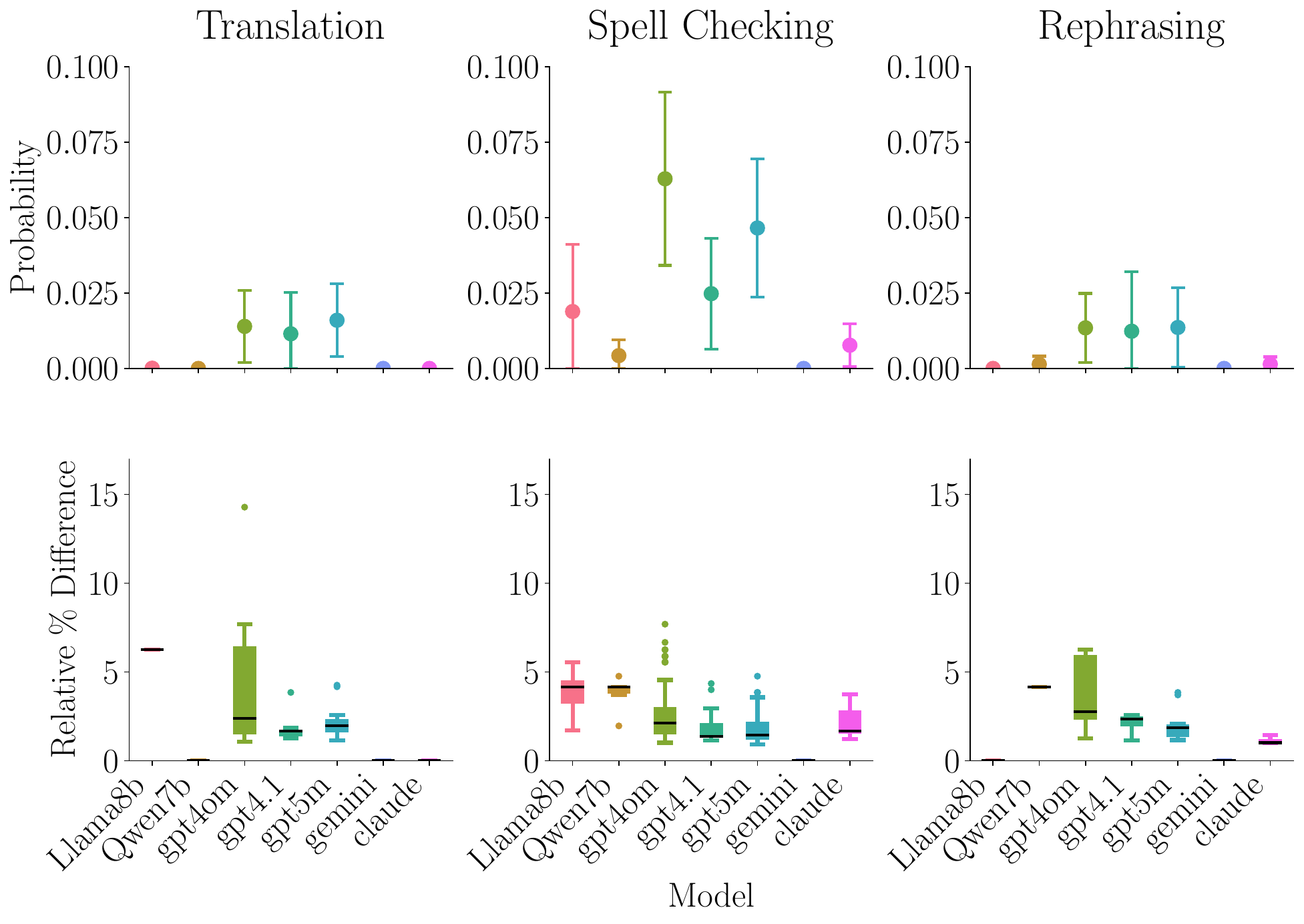}
        \caption{Swahili}
        \label{fig:multiplicity-all-sw}
    \end{subfigure}
    
    \caption{\textbf{Probability of tokenization multiplicity and magnitude of price variation for tasks in the (a) Turkish and (b) Swahili languages.}}
    \label{fig:multiplicity-all-tr-sw}
\end{figure}

\begin{figure}
    \centering
    \includegraphics[width=0.9\linewidth]{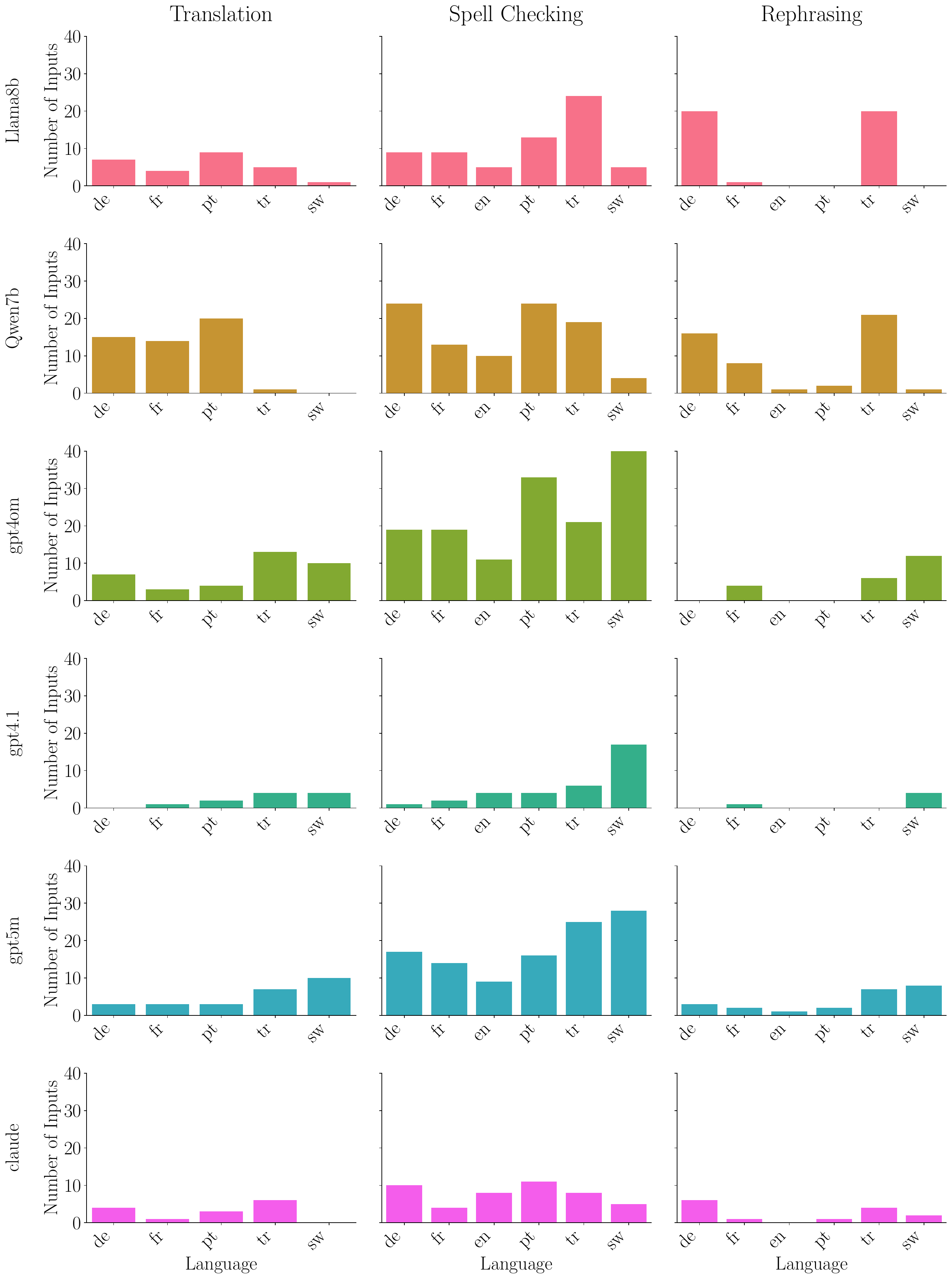}
    \caption{\textbf{Tokenization multiplicity across languages.}
  The plots show the number of inputs prompts for which we observe at least two outputs given by each LLM with the same string but different tokenization lengths.
  Each row corresponds to a different LLM, panels corresponds to one of the three tasks we consider in our experiments and pairs of letters on the x-axis correspond to different languages.}
    \label{fig:conflicts-all}
\end{figure}

\begin{figure}
    \centering
    \includegraphics[width=1\linewidth]{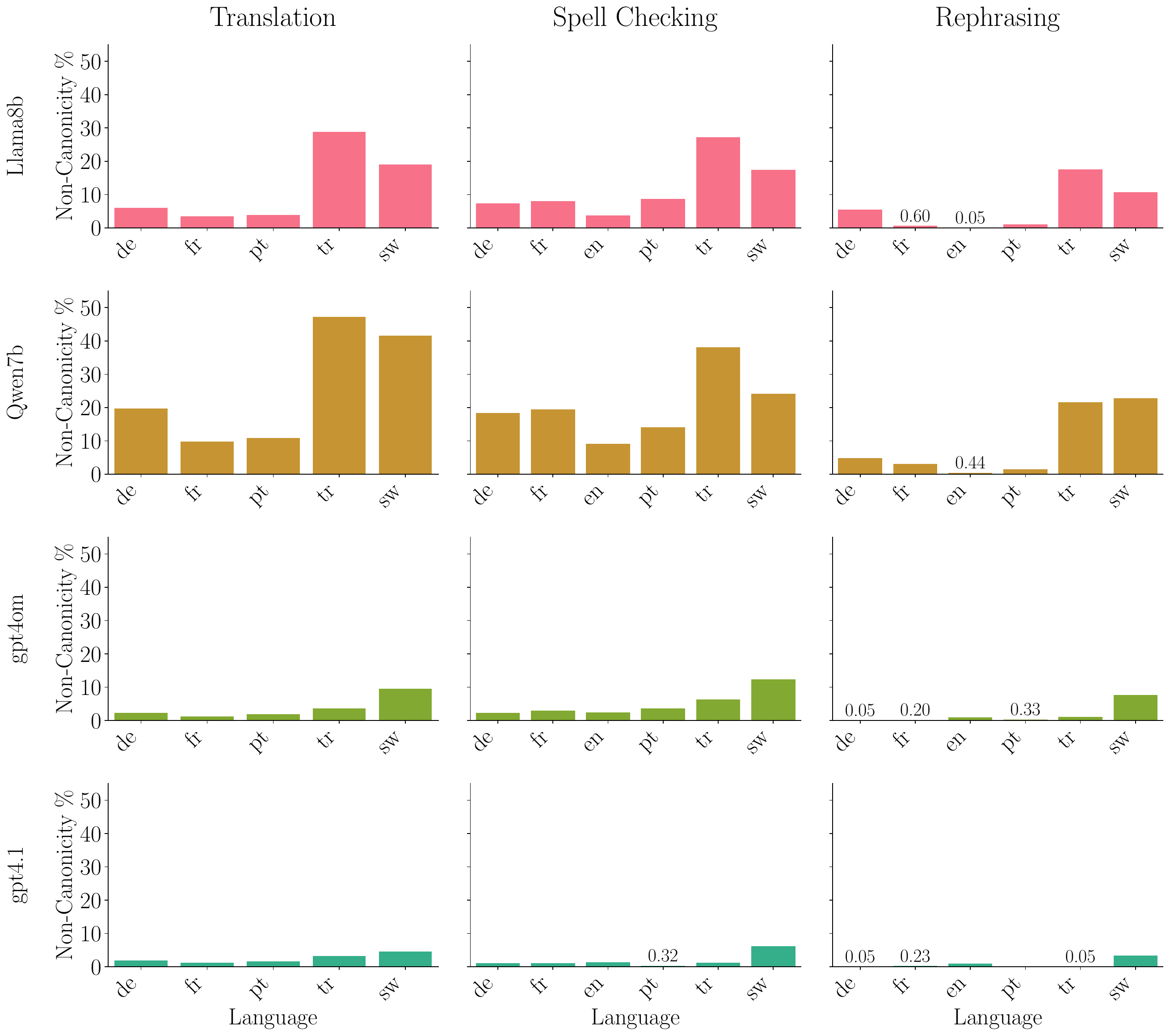}
    \caption{\textbf{Percentage of non-canonical outputs from LLMs that disclose the tokenization.}
    The plots show the percentage of outputs whose generated tokenization does not match the canonical tokenization of the same string.
    Each row corresponds to a different LLM, panels corresponds to one of the three tasks we consider in our experiments and pairs of letters on the x-axis correspond to different languages.}
    \label{fig:non-canon-transparent}
\end{figure}

\begin{figure}
    \centering
    \includegraphics[width=1\linewidth]{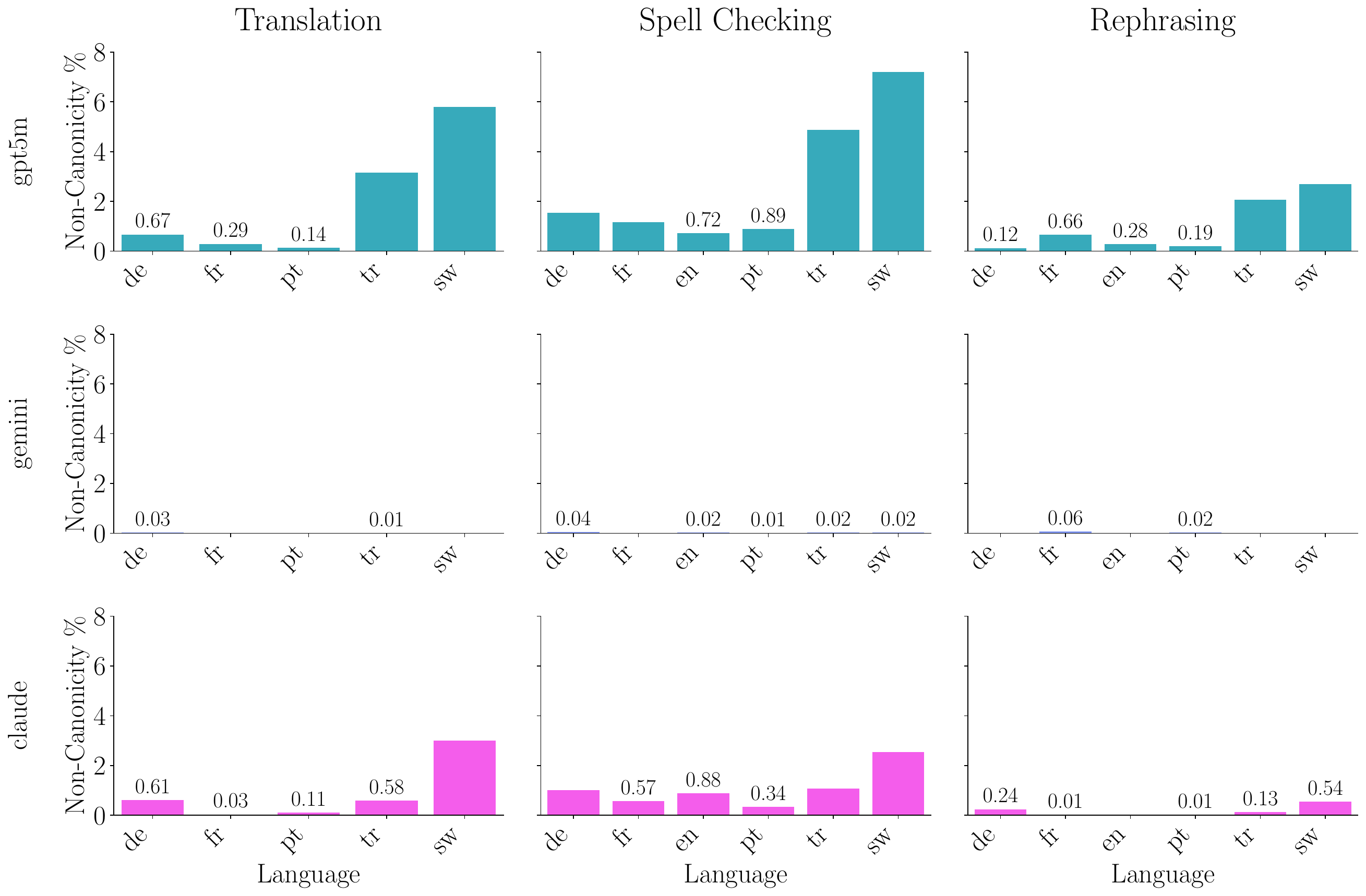}
    \caption{\textbf{Percentage of non-canonical outputs (lower bound) from LLMs that do not disclose the tokenization.}
    The plots show the percentage of outputs whose generated tokenization length does not match the canonical tokenization length for the same string.
    Each row corresponds to a different LLM, panels corresponds to one of the three tasks we consider in our experiments and pairs of letters on the x-axis correspond to different languages.}
    \label{fig:non-canon-nontransparent}
\end{figure}

\clearpage
\begin{figure}
    \centering
    \includegraphics[width=0.84\linewidth]{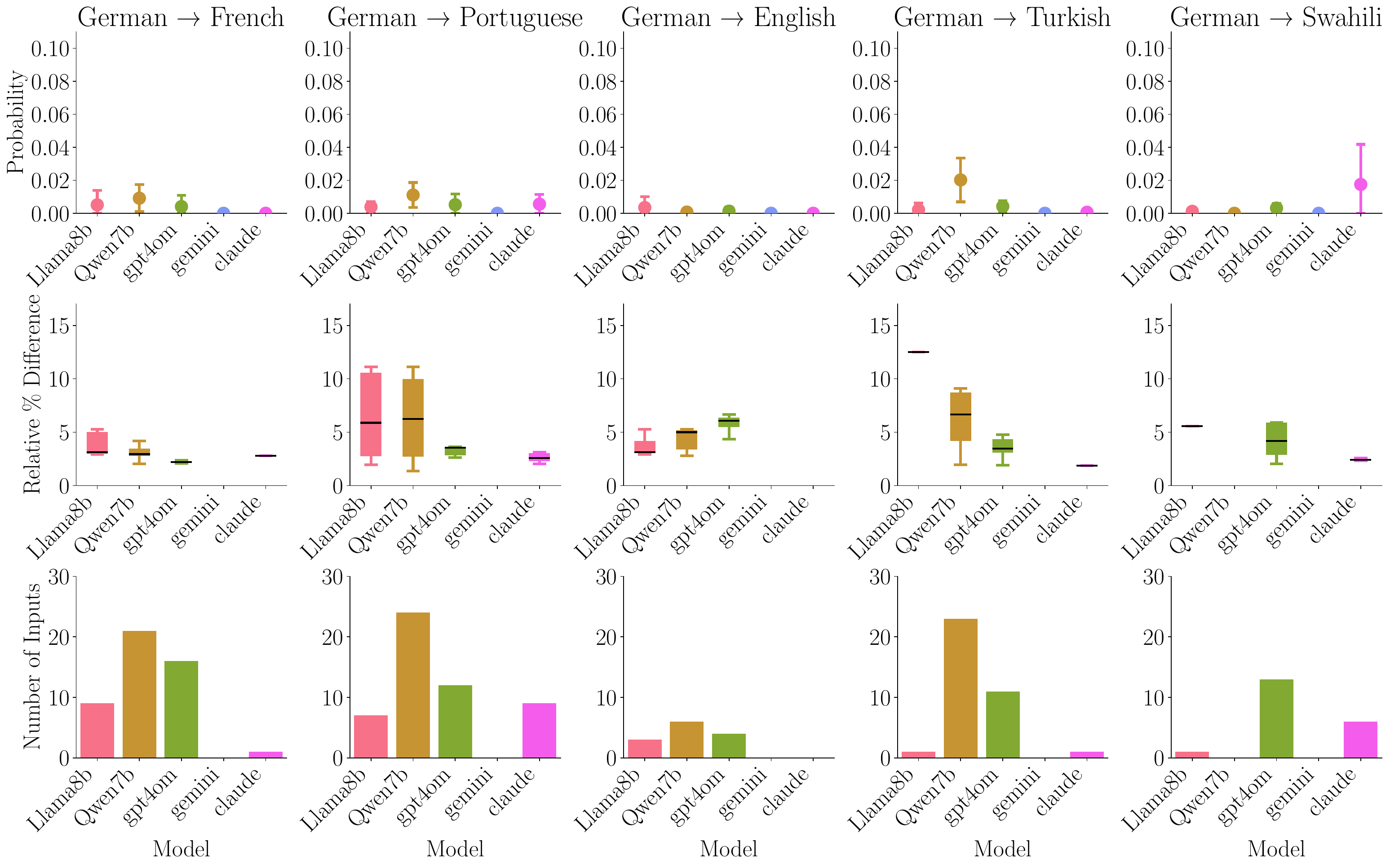}
    \caption{\textbf{Tokenization multiplicity in the translation task from German to other languages.}
    The top row shows the probability of tokenization multiplicity, the middle row shows the magnitude of price variation, and the bottom row shows the number of inputs where we observed tokenization multiplicity.
    Each column corresponds to a different target language.}
    \label{fig:translation-de-all}
\end{figure}

\vspace{-3mm}

\begin{figure}
    \centering
    \includegraphics[width=0.84\linewidth]{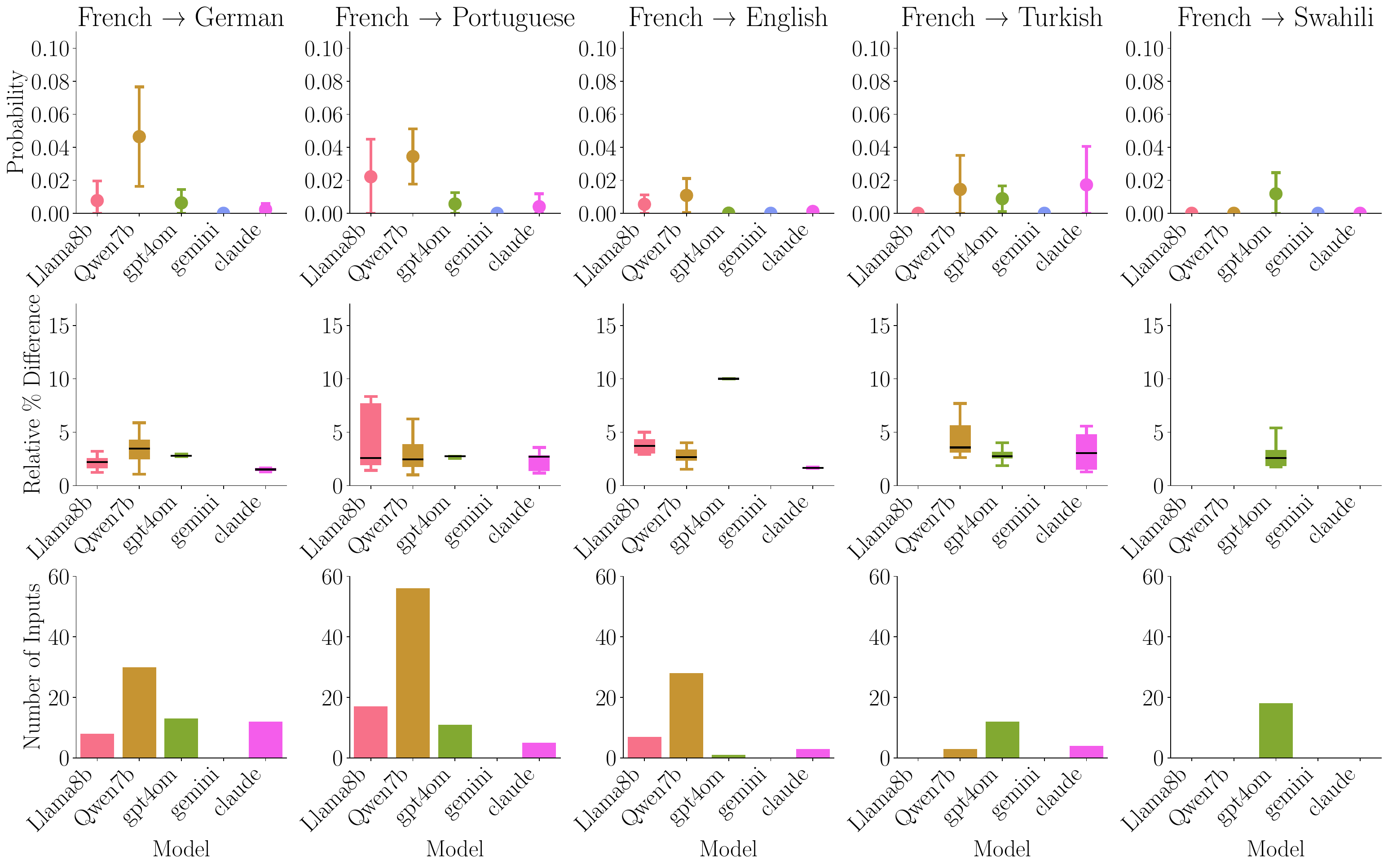}
    \caption{\textbf{Tokenization multiplicity in the translation task from French to other languages.}
    The top row shows the probability of tokenization multiplicity, the middle row shows the magnitude of price variation, and the bottom row shows the number of inputs where we observed tokenization multiplicity.
    Each column corresponds to a different target language.}
    \label{fig:translation-fr-all}
\end{figure}

\begin{figure}
    \centering
    \includegraphics[width=0.84\linewidth]{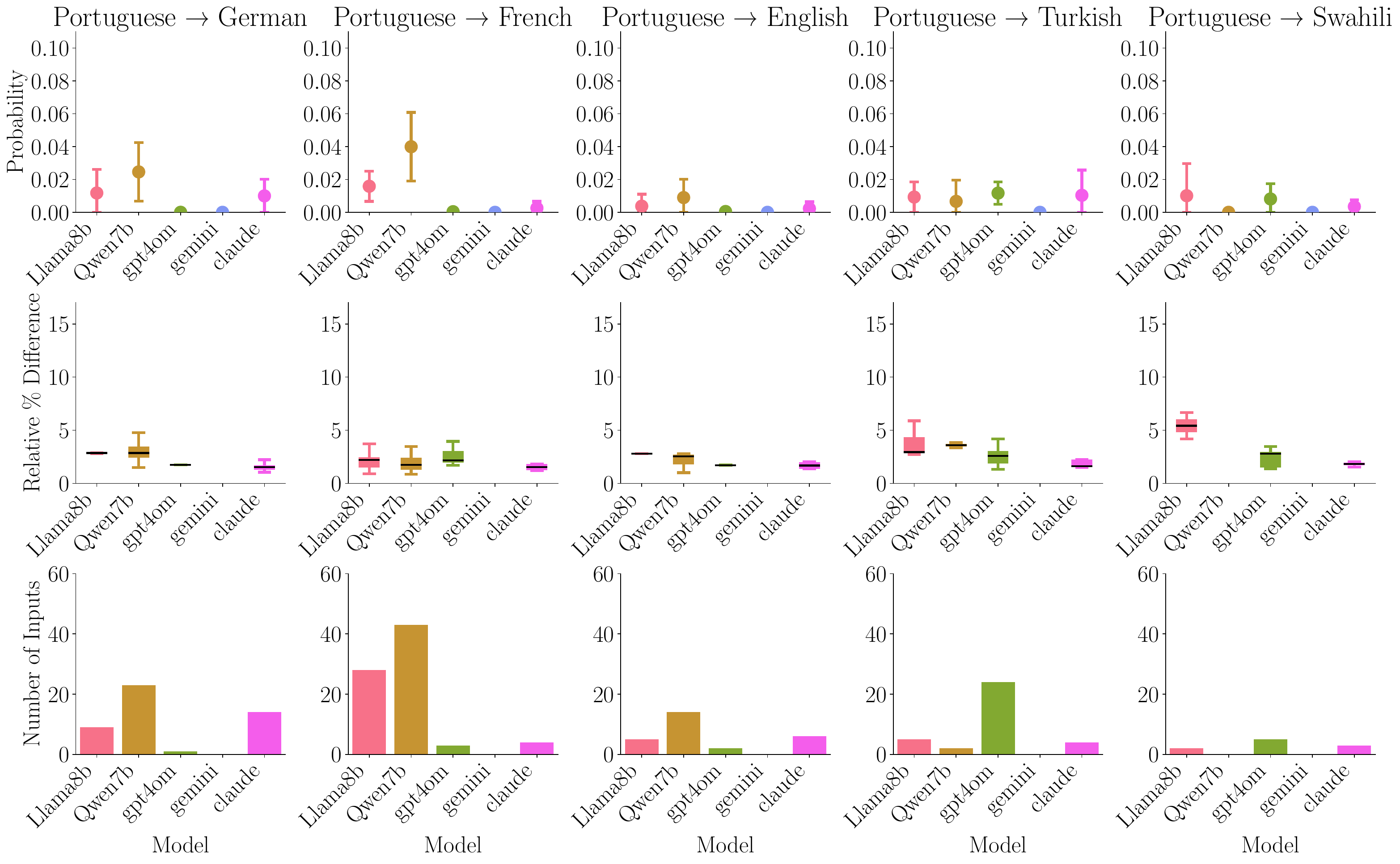}
    \caption{\textbf{Tokenization multiplicity in the translation task from Portuguese to other languages.}
    The top row shows the probability of tokenization multiplicity, the middle row shows the magnitude of price variation, and the bottom row shows the number of inputs where we observed tokenization multiplicity.
    Each column corresponds to a different target language.}
    \label{fig:translation-pt-all}
\end{figure}
\vspace{-3mm}

\begin{figure}
    \centering
    \includegraphics[width=0.84\linewidth]{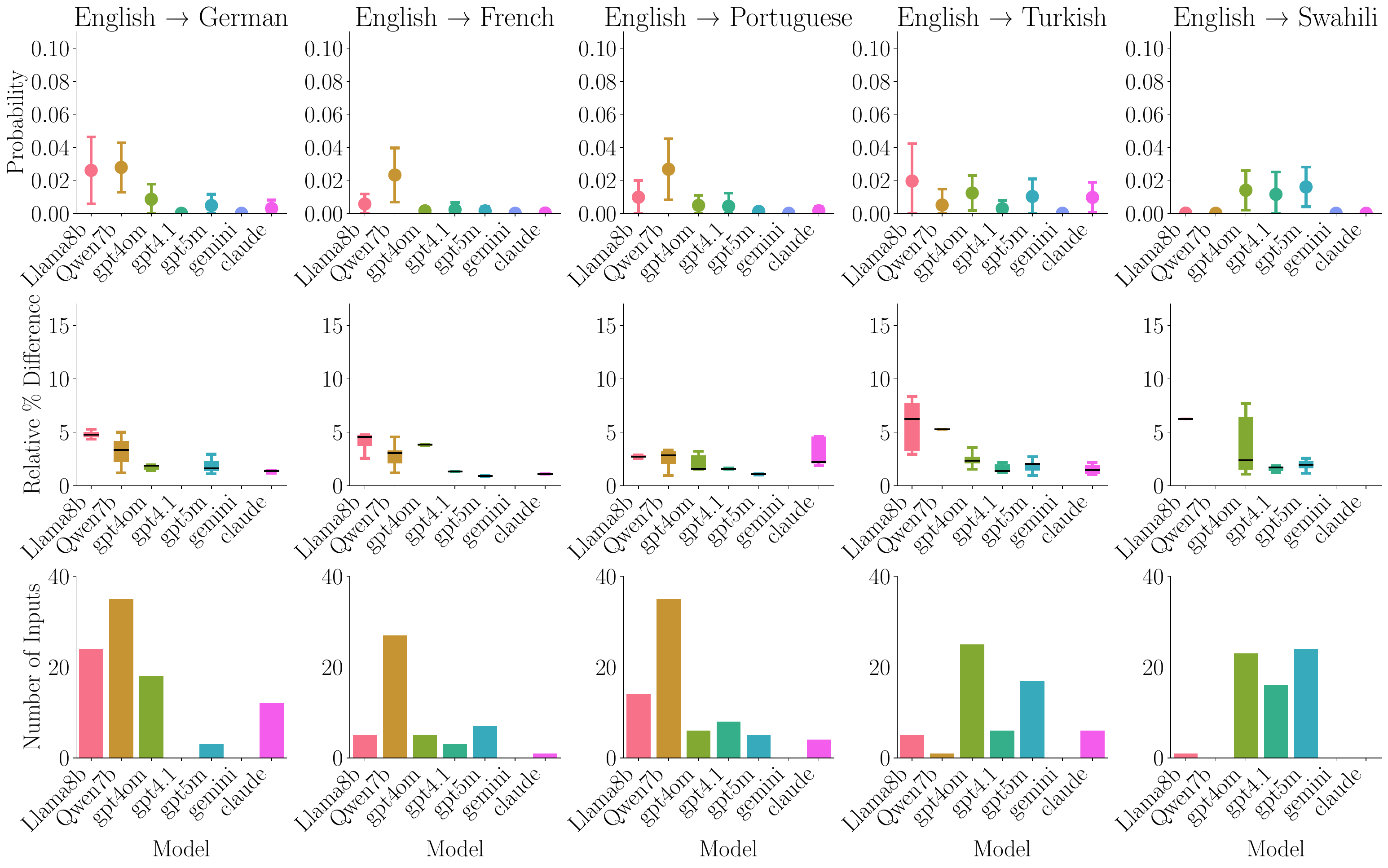}
    \caption{\textbf{Tokenization multiplicity in the translation task from English to other languages.}
    The top row shows the probability of tokenization multiplicity, the middle row shows the magnitude of price variation, and the bottom row shows the number of inputs where we observed tokenization multiplicity.
    Each column corresponds to a different target language.}
    \label{fig:translation-en-all}
\end{figure}

\begin{figure}
    \centering
    \includegraphics[width=0.84\linewidth]{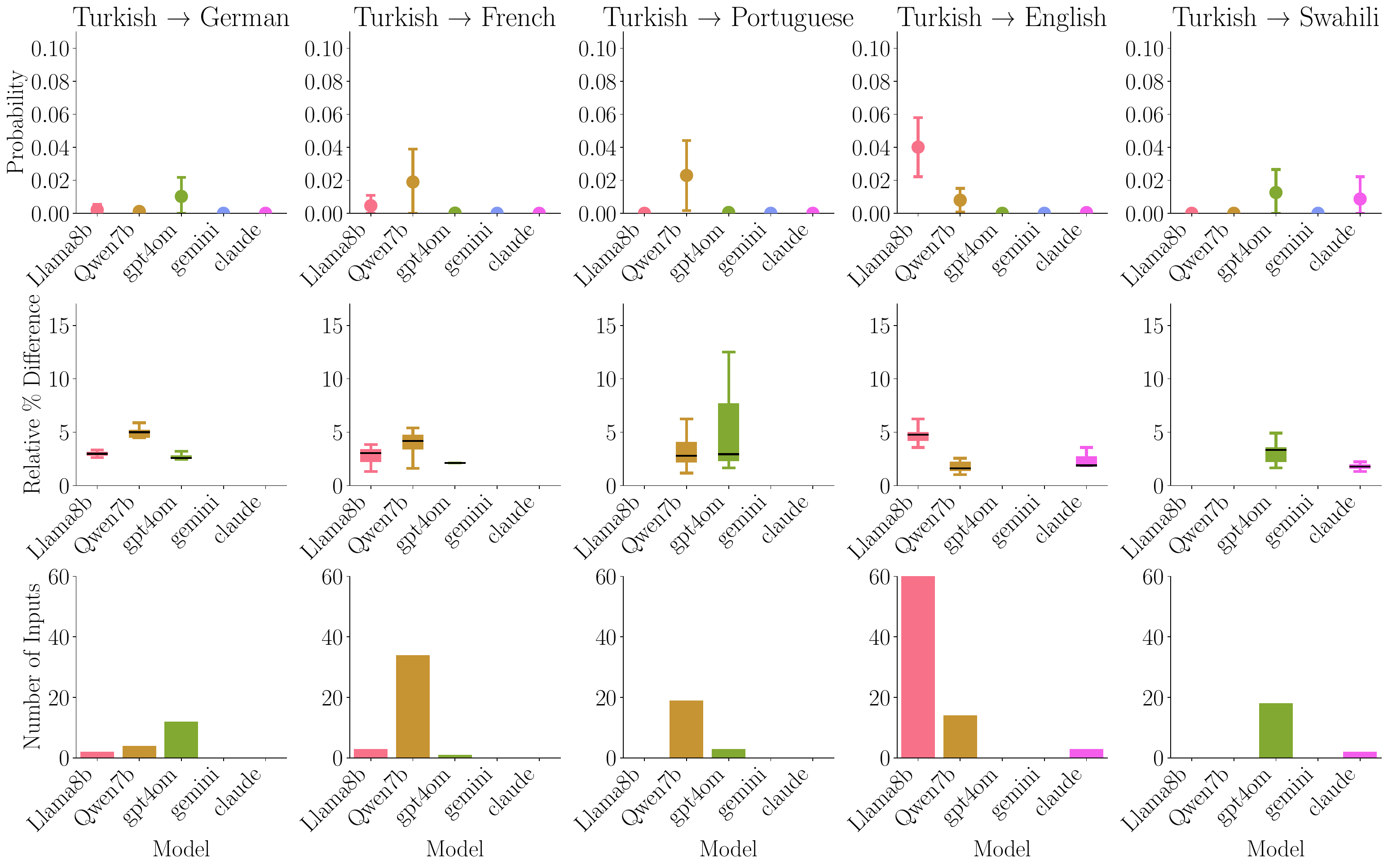}
    \caption{\textbf{Tokenization multiplicity in the translation task from Turkish to other languages.}
    The top row shows the probability of tokenization multiplicity, the middle row shows the magnitude of price variation, and the bottom row shows the number of inputs where we observed tokenization multiplicity.
    Each column corresponds to a different target language.}
    \label{fig:translation-tr-all}
\end{figure}
\vspace{-3mm}
\begin{figure}
    \centering
    \includegraphics[width=0.84\linewidth]{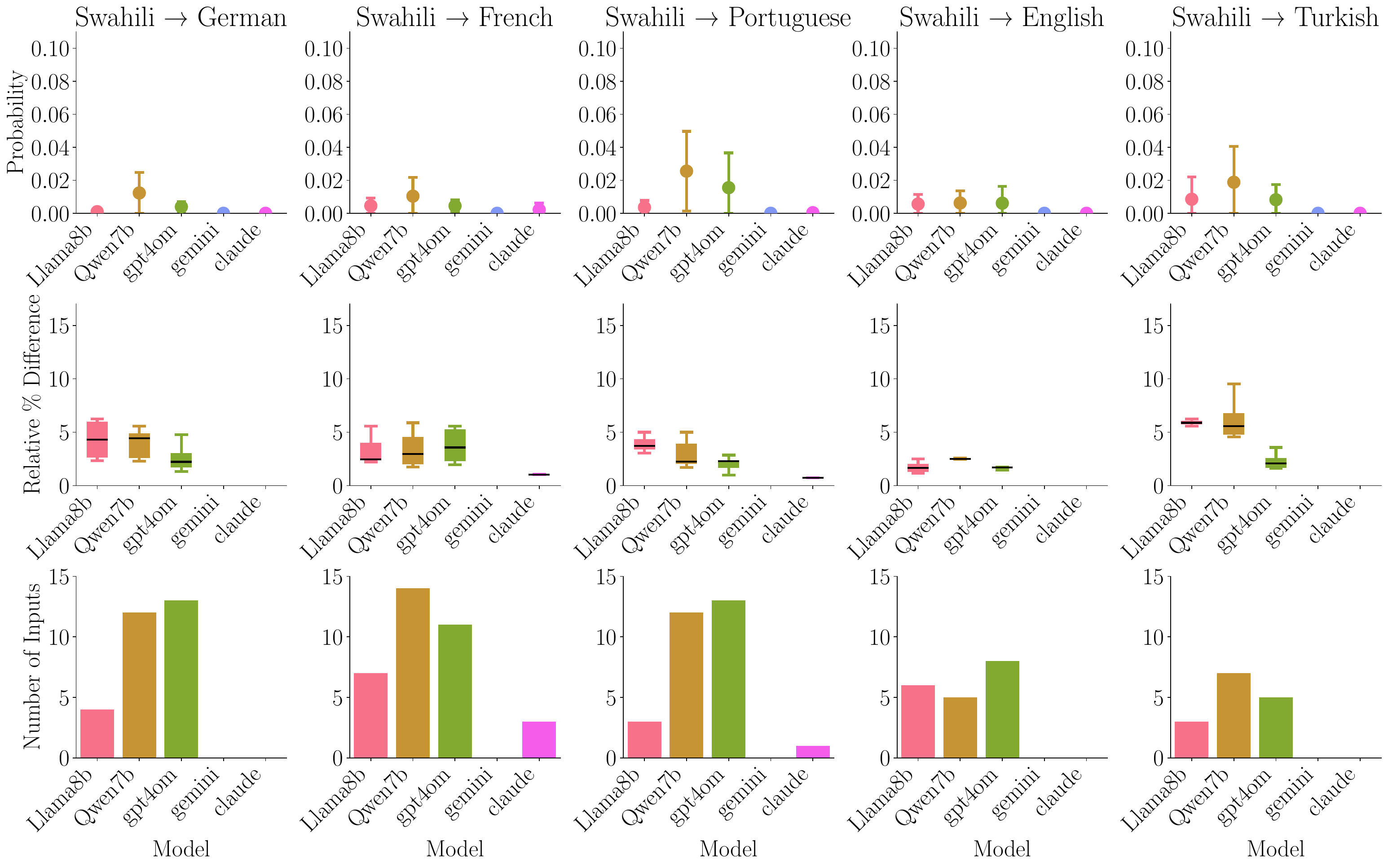}
    \caption{\textbf{Tokenization multiplicity in the translation task from Swahili to other languages.}
    The top row shows the probability of tokenization multiplicity, the middle row shows the magnitude of price variation, and the bottom row shows the number of inputs where we observed tokenization multiplicity.
    Each column corresponds to a different target language.}
    \label{fig:translation-sw-all}
\end{figure}

\clearpage
\subsection{Additional Experimental Results on Canonical Generation}
\label{app:results-canonical}
\vspace{-3mm}

\begin{table}[h!]
    \centering
    \small
    \begin{tabular}{l l c c c c}
    \toprule
    Language & Metric 
    & \multicolumn{2}{c}{Llama8B} 
    & \multicolumn{2}{c}{Qwen7B} \\
    \cmidrule(lr){3-4} \cmidrule(lr){5-6}
     &  & Standard & Canonical & Standard & Canonical \\
    \midrule
    \multirow{3}{*}{German}
        & Quality score
        & $0.72 \pm 0.02$ & $0.70 \pm 0.02$
        & $0.73 \pm 0.01$ & $0.71 \pm 0.01$ \\
        & Time per token (s)
        & $0.019$ & $0.020$ & $0.018$ & $0.019$ \\
        & Non-canonicity rate
        & $6.1\%$ & - & $17.6\%$ & - \\
    \midrule
    \multirow{3}{*}{French}
        & Quality score
        & $0.76 \pm 0.02$ & $0.74 \pm 0.02$
        & $0.78 \pm 0.02$ & $0.76 \pm 0.02$ \\
        & Time per token (s)
        & $0.020$ & $0.020$ & $0.018$ & $0.018$ \\
        & Non-canonicity rate
        & $3.5\%$ & - & $8.6\%$ & - \\
    \midrule
    \multirow{3}{*}{Portuguese}
        & Quality score
        & $0.72 \pm 0.03$ & $0.70 \pm 0.03$
        & $0.78 \pm 0.02$ & $0.76 \pm 0.02$ \\
        & Time per token (s)
        & $0.020$ & $0.020$ & $0.018$ & $0.018$ \\
        & Non-canonicity rate
        & $3.9\%$ & - & $10.0\%$ & - \\
    \midrule
    \multirow{3}{*}{Turkish}
        & Quality score
        & $0.57 \pm 0.02$ & $0.57 \pm 0.02$
        & $0.61 \pm 0.02$ & $0.60 \pm 0.01$ \\
        & Time per token (s)
        & $0.020$ & $0.020$ & $0.018$ & $0.019$ \\
        & Non-canonicity rate
        & $27.6\%$ & - & $44.4\%$ & - \\
    \midrule
    \multirow{3}{*}{Swahili}
        & Quality score
        & $0.60 \pm 0.01$ & $0.59 \pm 0.01$
        & $0.43 \pm 0.01$ & $0.43 \pm 0.01$ \\
        & Time per token (s)
        & $0.022$ & $0.020$ & $0.018$ & $0.019$ \\
        & Non-canonicity rate
        & $18.6\%$ & - & $37.2\%$ & - \\
    \bottomrule
    \end{tabular}
    \caption{{\bf Performance, (generation) time per token, and non-canonicity rate on the translation task.}
    The results comprise pairs of outputs generated with standard and canonical generation in all languages under the same source of randomness.
    For the time per token, confidence intervals were all smaller than $10^{-4}$.
    }
    \label{tab:eval-translate}
\end{table}

\vspace{-3mm}

\begin{table}[h!]
    \centering
    
    \small
    \begin{tabular}{l l c c c c}
    \toprule
    Language & Metric 
    & \multicolumn{2}{c}{Llama8B} 
    & \multicolumn{2}{c}{Qwen7B} \\
    \cmidrule(lr){3-4} \cmidrule(lr){5-6}
     &  & Standard & Canonical & Standard & Canonical \\
    \midrule
    \multirow{3}{*}{German}
        & $1$ $-$ edit distance
        & $0.62 \pm 0.04$ & $0.61 \pm 0.04$
        & $0.74 \pm 0.02$ & $0.72 \pm 0.02$ \\
        & Time per token (s)
        & $0.019$ & $0.023$ & $0.018$ & $0.018$ \\
        & Non-canonicity rate
        & $10.4\%$ & - & $19.0\%$ & - \\
    \midrule
    \multirow{3}{*}{French}
        & $1$ $-$ edit distance
        & $0.65 \pm 0.04$ & $0.64 \pm 0.04$
        & $0.77 \pm 0.02$ & $0.76 \pm 0.02$ \\
        & Time per token (s)
        & $0.019$ & $0.022$ & $0.018$ & $0.018$ \\
        & Non-canonicity rate
        & $11.8\%$ & - & $18.4\%$ & - \\
    \midrule
    \multirow{3}{*}{Portuguese}
        & $1$ $-$ edit distance
        & $0.81 \pm 0.03$ & $0.72 \pm 0.04$
        & $0.80 \pm 0.02$ & $0.76 \pm 0.02$ \\
        & Time per token (s)
        & $0.020$ & $0.022$ & $0.018$ & $0.019$ \\
        & Non-canonicity rate
        & $11.9\%$ & - & $16.8\%$ & - \\
    \midrule
    \multirow{3}{*}{English}
        & $1$ $-$ edit distance
        & $0.69 \pm 0.06$ & $0.68 \pm 0.06$
        & $0.85 \pm 0.02$ & $0.83 \pm 0.02$ \\
        & Time per token (s)
        & $0.019$ & $0.020$ & $0.018$ & $0.019$ \\
        & Non-canonicity rate
        & $5.5\%$ & - & $8.6\%$ & - \\
    \midrule
    \multirow{3}{*}{Turkish}
        & $1$ $-$ edit distance
        & $0.63 \pm 0.03$ & $0.62 \pm 0.03$
        & $0.70 \pm 0.02$ & $0.68 \pm 0.02$ \\
        & Time per token (s)
        & $0.020$ & $0.020$ & $0.020$ & $0.019$ \\
        & Non-canonicity rate
        & $26.8\%$ & - & $33.4\%$ & - \\
    \midrule
    \multirow{3}{*}{Swahili}
        &$1$ $-$ edit distance
        & $0.69 \pm 0.03$ & $0.68 \pm 0.03$
        & $0.74 \pm 0.02$ & $0.74 \pm 0.02$ \\
        & Time per token (s)
        & $0.020$ & $0.020$ & $0.021$ & $0.019$ \\
        & Non-canonicity rate
        & $17.4\%$ & - & $19.4\%$ & - \\
    \bottomrule
    \end{tabular}
    \caption{{\bf Performance, (generation) time per token, and non-canonicity rate on the spell checking task.}
    The results comprise pairs of outputs generated with standard and canonical generation in all languages under the same source of randomness.
    For the time per token, confidence intervals were all smaller than $10^{-4}$.
    }
    \label{tab:eval-typos}
\end{table}

\begin{table}[h!]
    \centering
    \small
    \begin{tabular}{l l c c c c}
    \toprule
    Language & Metric 
    & \multicolumn{2}{c}{Llama8B} 
    & \multicolumn{2}{c}{Qwen7B} \\
    \cmidrule(lr){3-4} \cmidrule(lr){5-6}
     &  & Standard & Canonical & Standard & Canonical \\
    \midrule
    \multirow{3}{*}{German}
        & Cosine similarity
        & $0.84 \pm 0.02$ & $0.84 \pm 0.02$
        & $0.89 \pm 0.02$ & $0.88 \pm 0.02$ \\
        & Time per token (s)
        & $0.020$ & $0.020$ & $0.018$ & $0.020$ \\
        & Non-canonicity rate
        & $5.7\%$ & - & $4.9\%$ & - \\
    \midrule
    \multirow{3}{*}{French}
        & Cosine similarity
        & $0.90 \pm 0.02$ & $0.90 \pm 0.02$
        & $0.93 \pm 0.02$ & $0.92 \pm 0.02$ \\
        & Time per token (s)
        & $0.019$ & $0.021$ & $0.018$ & $0.020$ \\
        & Non-canonicity rate
        & $2.2\%$ & - & $3.7\%$ & - \\
    \midrule
    \multirow{3}{*}{Portuguese}
        & Cosine similarity
        & $0.92 \pm 0.02$ & $0.91 \pm 0.02$
        & $0.96 \pm 0.01$ & $0.96 \pm 0.01$ \\
        & Time per token (s)
        & $0.020$ & $0.020$ & $0.018$ & $0.019$ \\
        & Non-canonicity rate
        & $2.2\%$ & - & $3.0\%$ & - \\
    \midrule
    \multirow{3}{*}{English}
        & Cosine similarity
        & $0.93 \pm 0.03$ & $0.90 \pm 0.04$
        & $0.96 \pm 0.02$ & $0.96 \pm 0.02$ \\
        & Time per token (s)
        & $0.019$ & $0.020$ & $0.018$ & $0.020$ \\
        & Non-canonicity rate
        & $0.2\%$ & - & $0.7\%$ & - \\
    \midrule
    \multirow{3}{*}{Turkish}
        & Cosine similarity
        & $0.87 \pm 0.01$ & $0.88 \pm 0.01$
        & $0.91 \pm 0.01$ & $0.90 \pm 0.01$ \\
        & Time per token (s)
        & $0.020$ & $0.020$ & $0.019$ & $0.019$ \\
        & Non-canonicity rate
        & $17.0\%$ & - & $17.0\%$ & - \\
    \midrule
    \multirow{3}{*}{Swahili}
        & Cosine similarity
        & $0.85 \pm 0.01$ & $0.85 \pm 0.01$
        & $0.85 \pm 0.01$ & $0.84 \pm 0.01$ \\
        & Time per token (s)
        & $0.022$ & $0.020$ & $0.018$ & $0.020$ \\
        & Non-canonicity rate
        & $10.3\%$ & - & $18.0\%$ & - \\
    \bottomrule
    \end{tabular}
    \caption{{\bf Performance, (generation) time per token, and non-canonicity rate on the rephrasing task.}
    The results comprise pairs of outputs generated with standard and canonical generation in all languages under the same source of randomness.
    For the time per token, confidence intervals are not shown, as they were all smaller than $10^{-4}$.
    }
    \label{tab:eval-rephrase}

\end{table}

\begin{table}[h!]
    \centering
    \small
    \begin{tabular}{l l c c c c}
    \toprule
    Language & Metric 
    & \multicolumn{2}{c}{Llama8B} 
    & \multicolumn{2}{c}{Qwen7B} \\
    \cmidrule(lr){3-4} \cmidrule(lr){5-6}
     &  & Standard & Canonical & Standard & Canonical \\
    \midrule
    \multirow{3}{*}{German}
        & Accuracy
        & $0.37 \pm 0.06$ & $0.37 \pm 0.06$
        & $0.63 \pm 0.05$ & $0.62 \pm 0.06$ \\
        & Time per token (s)
        & $0.020$ & $0.020$ & $0.018$ & $0.020$ \\
        & Non-canonicity rate
        & $22.3\%$ & - & $29.4\%$ & - \\
    \midrule
    \multirow{3}{*}{French}
        & Accuracy
        & $0.51 \pm 0.06$ & $0.50 \pm 0.06$
        & $0.22 \pm 0.15$ & $0.22 \pm 0.15$ \\
        & Time per token (s)
        & $0.020$ & $0.020$ & $0.019$ & $0.019$ \\
        & Non-canonicity rate
        & $13.0\%$ & - & $0.6\%$ & - \\
    \midrule
    \multirow{3}{*}{English}
        & Accuracy
        & $0.47 \pm 0.09$ & $0.47 \pm 0.09$
        & $0.86 \pm 0.12$ & $0.86 \pm 0.13$ \\
        & Time per token (s)
        & $0.020$ & $0.020$ & $0.019$ & $0.019$ \\
        & Non-canonicity rate
        & $4.2\%$ & - & $0.7\%$ & - \\
    \midrule
    \multirow{3}{*}{Spanish}
        & Accuracy
        & $0.57 \pm 0.06$ & $0.56 \pm 0.06$
        & $0.58 \pm 0.07$ & $0.58 \pm 0.07$ \\
        & Time per token (s)
        & $0.020$ & $0.020$ & $0.019$ & $0.019$ \\
        & Non-canonicity rate
        & $14.1\%$ & - & $5.2\%$ & - \\
    \midrule
    \multirow{3}{*}{Swahili}
        & Accuracy
        & $0.31 \pm 0.05$ & $0.32 \pm 0.05$
        & $0.13 \pm 0.03$ & $0.13 \pm 0.03$ \\
        & Time per token (s)
        & $0.021$ & $0.020$ & $0.018$ & $0.019$ \\
        & Non-canonicity rate
        & $32.7\%$ & - & $42.5\%$ & - \\
    \bottomrule
    \end{tabular}
\caption{{\bf Performance, (generation) time per token, and non-canonicity rate on the MGSM task.}
    The results comprise pairs of outputs generated with standard and canonical generation in all languages under the same source of randomness.
    For the time per token, confidence intervals are not shown, as they were all smaller than $10^{-4}$.
    }
    \label{tab:eval-mgsm}
    
\end{table}